\documentclass{article}
\usepackage{kr}

\usepackage{times}
\usepackage{soul}
\usepackage{url}
\usepackage[hidelinks]{hyperref}
\usepackage[utf8]{inputenc}
\usepackage[small]{caption}
\usepackage{graphicx}
\usepackage{amsmath}
\usepackage{amsthm}
\usepackage{booktabs}
\usepackage{algorithm}
\usepackage{algorithmic}
\title{Who's the Expert? On Multi-source Belief Change}

\author{%
Joseph Singleton$^1$\and
Richard Booth$^1$ \\
\affiliations
$^1$Cardiff University\\
\emails
\{singletonj1,boothr2\}@cardiff.ac.uk
}

\usepackage[inline]{enumitem}
\usepackage{xcolor}
\usepackage{amsmath}
\usepackage{amsthm}
\usepackage{amssymb}
\usepackage{environ}
\usepackage[capitalize]{cleveref}
\usepackage{scrextend}

\theoremstyle{plain}\newtheorem{example}{Example}
\theoremstyle{plain}\newtheorem{theorem}{Theorem}
\theoremstyle{plain}\newtheorem{proposition}{Proposition}
\theoremstyle{plain}\newtheorem{lemma}{Lemma}
\theoremstyle{plain}\newtheorem{definition}{Definition}
\theoremstyle{plain}
\theoremstyle{plain}\newtheorem{claim}{Claim}
\theoremstyle{plain}\newtheorem{corollary}{Corollary}

\newtheoremstyle{postulatestyle}{}{}{\slshape}{}{\bfseries}{}{ }{\thmnote{ #3 }}
\theoremstyle{postulatestyle} \newtheorem*{postulate}{}

\renewcommand{\S}{\mathcal{S}}
\newcommand{\C}{\mathcal{C}}
\newcommand{\W}{\mathcal{W}}
\newcommand{\X}{\mathcal{X}}
\newcommand{\Y}{\mathcal{Y}}
\newcommand{\propvars}{\mathcal{P}}
\newcommand{\lang}{\mathcal{L}}
\newcommand{\lprop}{\lang{}_0}
\newcommand{\lext}{\lang{}}
\newcommand{\vals}{\mathcal{V}}

\newcommand{\falsum}{\bot}
\newcommand{\Ninf}{\mathbb{N} \cup \{\infty\}}
\renewcommand{\phi}{\varphi}
\newcommand{\rs}{\upharpoonright}
\newcommand{\snd}{\mathsf{snd}}

\newcommand{\cn}{\operatorname{Cn}}
\newcommand{\mods}{\operatorname{mod}}
\renewcommand{\mod}{\mods}
\newcommand{\cnprop}{\cn_0}
\newcommand{\propmods}{\mods_0}
\newcommand{\concat}{\cdot}
\newcommand{\argmin}{\operatorname{argmin}}
\newcommand{\proppart}[1]{\left[{#1}\right]}

\newcommand{\postulatename}[1]{\emph{#1}}
\newcommand{\closure}{\postulatename{Closure}}
\newcommand{\containment}{\postulatename{Containment}}
\newcommand{\consistency}{\postulatename{Consistency}}
\newcommand{\soundness}{\postulatename{Soundness}}
\newcommand{\kbound}{\postulatename{K-bound}}
\newcommand{\priorext}{\postulatename{Prior-Extension}}
\newcommand{\kconj}{\postulatename{K-conjunction}}
\newcommand{\rearr}{\postulatename{Rearrangement}}
\newcommand{\equivpost}{\postulatename{Equivalence}}
\newcommand{\duprem}{\postulatename{Duplicate-removal}}
\newcommand{\condcons}{\postulatename{Conditional-consistency}}
\newcommand{\incvac}{\postulatename{Inclusion-vacuity}}
\newcommand{\acyc}{\postulatename{Acyc}}
\newcommand{\agm}{\postulatename{AGM-$\ast$}}
\newcommand{\condsucc}{\postulatename{Cond-success}}
\newcommand{\strongcondsucc}{\postulatename{Strong-cond-success}}
\newcommand{\boundedness}{\postulatename{Boundedness}}
\newcommand{\hboundedness}{\postulatename{H-Boundedness}}
\newcommand{\refinement}{\postulatename{Refinement}}

\newcommand{\concopname}[1]{$\mathsf{#1}$}
\newcommand{\weakop}{\concopname{weak\text{-}mb}}
\newcommand{\varbasedcond}{\concopname{var\text{-}based\text{-}cond}}
\newcommand{\partbasedcond}{\concopname{part\text{-}based\text{-}cond}}
\newcommand{\scorebasedop}{\concopname{excess\text{-}min}}

\newcommand{\tuple}[1]{\langle {#1} \rangle}
\newlist{inlinelist}{enumerate*}{1}
\setlist[inlinelist]{label=(\roman*)}

\newcommand{\citet}[1]{\citeauthor{#1}~\shortcite{#1}}

\def\thisistheprerint{}

\begin{document}

\maketitle

\begin{abstract}
    Consider the following belief change/merging scenario. A group of
information sources gives a sequence of reports about the state of the world at
various instances (e.g. different points in time). The true states at
these instances are unknown to us. The sources have varying levels of
expertise, also unknown to us, and may be knowledgeable on some topics but not
others. This may cause sources to report false statements in areas they lack
expertise. What should we believe on the basis of these reports? We provide a
framework in which to explore this problem, based on an extension of
propositional logic with expertise formulas. This extended language allows us
to express beliefs about the state of the world at each instance, as well as
beliefs about the expertise of each source. We propose several
postulates, provide a couple of families of concrete operators, and analyse
these operators with respect to the postulates.

\end{abstract}

\section{Introduction}

Consider the following belief change scenario in a hospital. We observe the
results of a blood test of patient 1, confirming condition X. Assuming the test
is reliable, the AGM paradigm~\cite{alchourron1985logic} tells us how to revise
our beliefs in light of the new information. Dr.\ A then claims that patient 2
suffers from the same condition, but Dr.\ B disagrees. Given that doctors
specialise in different areas and may make mistakes, who should we trust?
Since the \postulatename{Success} postulate ($\alpha \in K \ast \alpha$)
assumes information is reliable, we are outside the realm of AGM revision, and
must instead apply some form of \emph{non-prioritised}
revision~\cite{hansson1999survey}.

Suppose it now emerges that Dr.\ A had earlier claimed patient 1 did \emph{not}
suffer from condition X, contrary to the test results. We now have reason to
suspect Dr.\ A may \emph{lack expertise} on diagnosing X, and may subsequently
revise beliefs about Dr.\ A's domain of expertise and the status of patient
2 (e.g. by opting to trust Dr.\ B instead).

While simple, this example illustrates the key features of the belief change
problem we study: we consider multiple sources, whose expertise is \emph{a
priori} unknown, providing reports on various instances of a problem domain. On
the basis of these reports we form beliefs both about the expertise of the
sources and the state of the world in each instance.

By including a distinguished \emph{completely reliable} source (the test
results in the example) we extend AGM revision. In some respects we also extend
approaches to non-prioritised revision (e.g. selective
revision~\cite{ferme1999selective}, credibility-limited
revision~\cite{hansson_2001}, and trust-based
revision~\cite{booth_trust_2018}), which assume information about the
reliability of sources is known up front. The problem is also related to
\emph{belief merging}~\cite{konieczny2002merging} which deals
with combining belief bases from multiple sources; a detailed comparison will
be given in \cref{sec:relatedwork}.

Our work is also connected to trust and belief revision, if one interprets
trust as \emph{belief in expertise}. As \citet{yasser_21} note in recent work,
trust and belief are inexorably linked: we should accept reports from sources
we believe are trustworthy, and we should trust sources whose reports turn out
to be reliable. Trust and belief should also be revised in tandem, so that we
may increase or decrease trust in a source as more reports are received,
and revoke or reinstate previous reports from a source as its perceived
trustworthiness changes.\footnotemark{}

\footnotetext{
    This mutual dependence between trust and belief is also the core
    idea in \emph{truth discovery}~\cite{li_survey_2016}.
}

To unify the trust and belief aspects, we enrich a propositional language with
\emph{expertise statements} $E_i(\phi)$, read as ``source $i$ has expertise on
$\phi$". The output of our belief change problem is then a collection of belief
and knowledge sets in the extended language, describing what we \emph{know} and
\emph{believe} about the expertise of the sources and the state of the world in
each instance. For example, we should \emph{know} reports from the reliable
source are true, whereas reports from ordinary sources may only be believed.

Following recent work on logical approaches to
expertise~\cite{singleton2021logic,booth_trust_2018}, we formally model
expertise using a partition of propositional valuations for each source.
Equivalently, each source has an \emph{indistinguishibility equivalence
relation} over valuations. A source is an expert on a proposition $\phi$
exactly when they can distinguish every $\phi$ valuation from every
$\neg\phi$ valuation.\footnotemark{} As in \citet{singleton2021logic}, we also
use \emph{soundness formulas} $S_i(\phi)$, which intuitively say that $\phi$ is
true \emph{up the expertise of} $i$. For example, if $i$ has expertise on $p$
but not $q$, then the conjunction $p \land q$ is sound for $i$ whenever $p$
holds, since we can effectively ignore $q$. Formally,
$\phi$ is sound for $i$ if the ``actual" state of the world is
indistinguishable from a $\phi$ valuation. Note that expertise does not depend
on the ``actual" state, whereas soundness does. This provides a crucial link
between expertise and truthfulness of information.

We then make the assumption that \emph{sources only report sound propositions}.
That is, reports are only false due to sources overstepping the bounds of their
expertise. In particular, we assume sources are honest in their reports,
and that experts are always right.

\footnotetext{
    The relationship between this notion of expertise and \emph{S5 epistemic
    logic} is explored in a modal logic setting in \citet{singleton2021logic},
    and we revisit this connection in \cref{sec:relatedwork}.
}

Note that in our introductory example, the fact that we had a report from Dr.\ A
on patient 1 (together with reliable information on patient 1) was essential
for determining the expertise of Dr.\ A, and subsequently the status of patient
2. While the patients are independent, reports on one can cause beliefs about
the other to change, as we update our beliefs about the expertise of the
sources.

In general we consider an arbitrary number of \emph{cases}, which are
seen as labels for instances of the domain. For example, a crowdsourcing worker
may label multiple images, or a weather forecaster may give predictions for
different locations. Each report in the input to the
problem then refers to a specific case. Via these cases and the presence of the
completely reliable source, we are able to model scenarios where some ``ground
truth" is available, listing how often sources have been correct/incorrect on a
proposition (e.g. the \emph{report histories} of
\citet{hunter_building_21}). We can also generalise this scenario, e.g. by
having only partial information about ``previous" cases.

Throughout the paper we make the assumption that \emph{expertise is fixed
across cases}: the expertise of a source does not depend on the particular
instance of the domain we look at. For instance, the expertise of Dr.\ A is the
same for patient 1 as for patient 2. This is a simplifying assumption, and may
rule out certain interpretations of the cases (e.g. if cases represent
different points in time, it would be natural to let expertise evolve over
time).

\paragraph{Contribution.} Our contributions are threefold. First, we develop a
logical framework for reasoning about the expertise of multiple sources and the
state of the world in multiple cases. Second, we formulate a belief change
problem within this framework, which allows us to explore how trust and belief
should interact and evolve as reports are received from the various sources.
Finally, we put forward several postulates and two concrete classes of
operators -- with a representation result for one class -- and analyse these
operators with respect to the postulates.

\paragraph{Paper outline.} In \cref{sec:framework} we develop the formal
framework. \Cref{sec:the_problem} introduces the problem and lists some core
postulates. We give two constructions and specific example operators in
\cref{sec:constructions}. \Cref{sec:one_step_postulates} introduces some
further postulates concerning belief change on the basis of one new report. An
analogue of selective revision~\cite{ferme1999selective} is presented
\cref{sec:selective_change}. \Cref{sec:relatedwork} discusses related work, and
we conclude in \cref{sec:conclusion}.
\ifdefined\thisistheprerint
    Note that some proofs are omitted in the body of the paper and can be found
    in the appendix.
\else
    Proofs are given in the appendix of the full version of the
    paper~\cite{singleton_booth_22_preprint}.
\fi

\section{The Framework}
\label{sec:framework}

Let $\S$ be a finite set of information sources. For convenience, we assume
there is a \emph{completely reliable} source in $\S$, which we denote by
$\ast$.  For example, we can treat our first-hand observations as if they are
reported by $\ast$. Other sources besides $\ast$ will be termed \emph{ordinary
sources}.  Let $\C$ be a finite set of \emph{cases}, which we interpret as
labels for different instances of the problem domain.

\paragraph{Syntax.}

There are two levels to our formal language. To describe properties of the
world in each case $c \in \C$, we assume a fixed finite set $\propvars$ of propositional
variables, and let $\lprop$ denote the set of propositional formulas generated
from $\propvars$ using the usual propositional connectives. We use lower case
Greek letters ($\phi$, $\psi$ etc) for formulas in $\lprop$. The
classical logical consequence operator will be denoted by $\cnprop$, and
$\equiv$ denotes equivalence of propositional formulas.

The extended \emph{language of expertise} $\lext$ additionally describes the
expertise of the sources, and is defined by the following grammar:
\[
    \Phi ::= \phi \mid
             \Phi \land \Phi \mid
             \neg \Phi \mid
             E_i(\phi) \mid
             S_i(\phi)
             \\
\]
where $i\in \S$ and $\phi \in \lprop$. We introduce Boolean connectives
$\lor$, $\rightarrow$, $\leftrightarrow$ and $\falsum$ as
abbreviations. We use upper case Greek letters ($\Phi$, $\Psi$ etc)
for formulas in $\lext$.
For $\Gamma \subseteq \lext$, we write $\proppart{\Gamma} = \Gamma \cap
\lprop$ for the propositional formulas in $\Gamma$.

The intuitive reading of $E_i(\phi)$ is {\em source $i$ has expertise on
$\phi$}, i.e., $i$ is able to correctly identify the truth value of $\phi$ in
any possible state. The intuitive reading of $S_i(\phi)$ is that $\phi$
\emph{sound} for $i$ to report: that $\phi$ is true up to the
expertise of $i$. That is, the parts of $\phi$ on which $i$ has expertise are
true. Note that both operators are restricted to propositional
formulas, so we will not consider iterated formulas such as $E_i(S_j(\phi))$.

\paragraph{Semantics.}

Let $\vals$ denote the set of propositional valuations over $\propvars$. For
each $\phi \in \lprop$, the set of valuations making $\phi$ true is denoted by
$\propmods(\phi)$. A \emph{world} $W = \tuple{\{v_c\}_{c \in \C}, \{\Pi_i\}_{i
\in \S}}$ is a possible complete specification of the environment we find
ourselves in:
\begin{itemize}
    \item $v_c \in \vals$ is the ``true" valuation at case $c \in \C$;
    \item $\Pi_i$ is a partition of $\vals$ for each $i \in \S$, representing
          the ``true" expertise of source $i$; and
    \item $\Pi_\ast$ is the unit partition $\{\{v\} \mid v \in \vals\}$.
\end{itemize}
Let $\W$ denote the set of all worlds.  Note that the partition corresponding
to the distinguished source $\ast$ is fixed in all worlds as the finest
possible partition, reflecting the fact that $\ast$ is completely reliable.

For any partition $\Pi$ and valuation $v$, write $\Pi[v]$ for the unique cell
in $\Pi$ containing $v$. For a set of valuations $U$, write $\Pi[U] =
\bigcup_{v \in U}{\Pi[v]}$. For brevity, we write $\Pi[\phi]$ for
$\Pi[\propmods(\phi)]$. Then $\Pi[\phi]$ is the set of valuations
indistinguishable from a $\phi$ valuation.

For our belief change problem we will be interested in maintaining a
collection of several belief sets, describing beliefs about each case $c \in
\C$. Towards determining when a world $W$ models such a collection, we define
semantics for $\lext$ formulas with respect to a world and a case:
\[
    \begin{array}{lll}
        &W, c \models \phi &\iff v_c \in \propmods(\phi) \\
        &W, c \models E_i(\phi) &\iff \Pi_i[\phi] = \propmods(\phi) \\
        &W, c \models S_i(\phi) &\iff v_c \in \Pi_i[\phi]
    \end{array}
\]
where $i \in \S$, $\phi \in \lprop$, and the clauses for conjunction and
negation are the expected ones. Since $\propmods(\phi) \subseteq \Pi_i[\phi]$
always holds, we have that $E_i(\phi)$ holds iff there is no
$\neg\phi$ valuation which is indistinguishable from a $\phi$ valuation (c.f.
\citet{booth_trust_2018}). Note
that since each source $i$ has only a single partition $\Pi_i$ used to
interpret the expertise formulas, the truth value of $E_i(\phi)$ does not
depend on the case $c$. On the other hand, $S_i(\phi)$ holds in case $c$ iff
the $c$-valuation of $W$ is indistinguishable from some model of $\phi$. That
is, it is consistent with $i$'s expertise that $\phi$ is true.


Also note that if $\phi$ is a propositional tautology, $E_i(\phi)$ holds for
every source $i$. Thus, all sources are experts on \emph{something}, even if
just the tautologies.

\begin{example}
    \label{ex:hospital_world}
    Let us extend the hospital example from the introduction. Let $\S = \{\ast,
    a, b\}$ denote the reliable source, Dr.\ A and Dr.\ B, and let $\C = \{c_1,
    c_2\}$ denote patients 1 and 2. Consider propositional variables
    $\propvars = \{x, y\}$, standing
    for condition X and Y respectively. Suppose that Dr.\ A has expertise on
    diagnosing condition Y only, whereas Dr.\ B only has expertise on X. For the
    sake of the example, suppose that patient 1 suffers from both conditions,
    and patient 2 suffers only from condition Y. This situation is modelled by
    the following world $W = \tuple{\{v_c\}_{c \in \{c_1, c_2\}}, \{\Pi_i\}_{i
    \in \{\ast, a, b\}}}$:
    \[
        \begin{array}{cc}
            v_{c_1} = xy;
            &
            v_{c_2} = \bar{x}y;
            \\
            \Pi_a = xy, \bar{x}y \mid x\bar{y}, \bar{x}\bar{y};
            &
            \Pi_b = xy, x\bar{y} \mid \bar{x}y, \bar{x}\bar{y}.
        \end{array}
    \]
    We have $W, c \models E_a(y)
    \land E_b(x)$ for each $c \in \{c_1, c_2\}$. Also note that $W, c_1 \models
    x$ (patient 1 suffers from X), $W, c_1 \models S_a(\neg x)$ (it is sound
    for Dr.\ A to report otherwise; this holds since $\Pi_a[\neg x] = \{xy,
    \bar{x}y\} \cup \{x\bar{y}, \bar{x}\bar{y}\} \ni xy = v_{c_1}$), but $W,
    c_1 \models \neg S_b(\neg x)$ (the same formula is \emph{not} sound for
    Dr.\
    B; we have $\Pi_b[\neg x] = \{\bar{x}y, \bar{x}\bar{y}\} = \propmods(\neg
    x) \not\ni xy = v_{c_1}$).
\end{example}

Say $\Phi$ is \emph{valid} if $W, c \models \Phi$ for all $W \in \W$ and $c \in \C$.
For future reference we collect a list of validities.

\begin{proposition}
\label{prop:validities}
For any $i \in \S$, $c \in \C$ and $\phi, \psi \in \lprop$, the following
formulas are valid
\begin{enumerate}
    \item \label{item:replacement_equivalents_e_s}
          $S_i(\phi) \leftrightarrow S_i(\psi)$ and $E_i(\phi) \leftrightarrow
          E_i(\psi)$, whenever $\phi \equiv \psi$
    \item \label{item:e_symmetric}
          $E_i(\phi) \leftrightarrow E_i(\neg \phi)$ and $E_i(\phi) \land
          E_i(\psi) \rightarrow E_i(\phi \land \psi)$
    \item \label{item:exp_on_all_variables} $E_i(p_1) \land \cdots \land
          E_i(p_k) \rightarrow E_i(\phi)$, where $p_1, \ldots, p_k$ are the
          propositional variables appearing in $\phi$
    \item \label{item:e_and_s_implies_phi}
          $E_i(\phi) \land S_i(\phi) \rightarrow \phi$, and
          $S_i(\phi) \land \neg \phi \rightarrow \neg E_i(\phi)$
    \item \label{item:sound_neg_pair}
          $S_i(\phi) \land S_i(\neg \phi) \rightarrow \neg E_i(\phi)$
    \item \label{item:star_exp}
          $S_\ast(\phi) \leftrightarrow \phi$ and $E_\ast(\phi)$
\end{enumerate}
\end{proposition}

(\labelcref{item:replacement_equivalents_e_s}) states
syntax-irrelevance properties.
(\labelcref{item:e_symmetric}) says that expertise is symmetric with respect to
negation, and closed under conjunctions. Intuitively, symmetry means that $i$
is an expert on $\phi$ if they know \emph{whether or not} $\phi$ holds.
(\labelcref{item:exp_on_all_variables}) says that expertise on each
propositional variable in $\phi$ is sufficient for expertise on
$\phi$ itself.
(\labelcref{item:e_and_s_implies_phi}) says that, in the presence of expertise,
soundness of $\phi$ is sufficient for $\phi$ to in fact be true.
(\labelcref{item:sound_neg_pair})
says that if both $\phi$ and
$\neg\phi$ are true up to the expertise of $i$, then $i$ cannot have expertise
on $\phi$.
Finally, (\labelcref{item:star_exp}) says that the reliable source $\ast$ has expertise
on \emph{all} formulas, and thus $\phi$ is sound for $\ast$ iff it is true.


\paragraph{Case-indexed Collections.}

In the remainder of the paper we will be interested in forming beliefs about
each case $c \in \C$. To do so we use collections of belief
sets $G = \{\Gamma_c\}_{c \in \C}$, with $\Gamma_c \subseteq \lext$, indexed by
cases.
Say a world $W$ is a \emph{model} of $G$ iff
\[
    W, c \models \Phi \text{ for all } c \in \C \text{ and } \Phi \in \Gamma_c,
\]
i.e. iff $W$ satisfies all formulas in $G$ in the relevant case. Let
$\mods(G)$ denote the models of $G$, and say that $G$ is
\emph{consistent} if $\mods(G) \ne \emptyset$. For $c \in \C$, define the
\emph{$c$-consequences}
\[
    \cn_c(G) = \{\Phi \in \lext \mid \forall W \in \mods(G), W, c \models \Phi\}.
\]
We write $\cn(G)$
for the collection $\{\cn_c(G)\}_{c \in \C}$.

\begin{example}
    Suppose $\C = \{c_1, c_2, c_3\}$, and define $G$ by $\Gamma_{c_1} = \{S_i(p
    \land q)\}$, $\Gamma_{c_2} = \{E_i(p)\}$ and $\Gamma_{c_3} = \{E_i(q)\}$.
    Then, since expertise holds independently of case, any model $W$ of $G$ has
    $W, c_1 \models E_i(p) \land E_i(q)$. By \cref{prop:validities} part
    (\labelcref{item:exp_on_all_variables}), $W, c_1 \models E_i(p \land q)$.
    Since $W$ satisfies $\Gamma_{c_1}$ in case $c_1$, \cref{prop:validities}
    part (\labelcref{item:e_and_s_implies_phi}) gives $W, c_1 \models p \land
    q$. Since $W$ was an arbitrary model of $G$, we have $p \land q \in
    \cn_{c_1}(G)$, i.e. $p \land q$ is a $c_1$-consequence of $G$.
    This illustrates how information about distinct cases can be brought
    together to have consequences for other cases.

\end{example}

For two collections $G = \{\Gamma_c\}_{c \in \C}$, $D = \{\Delta_c\}_{c \in
\C}$, write $G \sqsubseteq D$ iff $\Gamma_c \subseteq \Delta_c$ for all $c$,
and let $G \sqcup D$ denote the collection $\{\Gamma_c \cup \Delta_c\}_{c \in
\C}$. With this notation, the case-indexed consequence operator satisfies
analogues of the Tarskian consequence properties.\footnotemark{}
\footnotetext{
    That is, \begin{inlinelist}
        \item $G \sqsubseteq \cn(G)$,
        \item $G \sqsubseteq D$ implies $\cn(G) \sqsubseteq \cn(D)$, and
        \item $\cn(\cn(G)) = \cn(G)$.
    \end{inlinelist}
}

Say a collection $G$ is \emph{closed} if $\cn(G) = G$. Closed collections
provide an idealised representation of beliefs, which will become useful later
on. For instance, when $G$ is closed we have
$
    E_i(\phi) \in \Gamma_c \text{ iff } E_i(\phi) \in \Gamma_d
$
for all $c, d \in \C$; i.e. expertise statements are either present for all
cases or for none. We also have
$
    \cnprop\proppart{\Gamma_c} = \proppart{\Gamma_c},
$
i.e. the propositional parts of $G$ are (classically) closed.

In propositional logic, $\propmods$ is a 1-to-1 correspondence between closed
sets of formulas and sets of valuations. This is not so in our setting, since
some subsets of $\W$ do not arise as the models of any collection. Instead, we
have a 1-to-1 correspondence into a restricted collection of sets of worlds.
Borrowing the terminology of
\citet{delgrande2018general}, say a set of worlds $S \subseteq \W$ is
\emph{elementary} if ${S = \mod(G)}$ for some collection $G =
\{\Gamma_c\}_{c \in \C}$.\footnotemark{}

\footnotetext{
    Non-elementary sets can also exist for weaker logics
    (such as Horn logic~\cite{delgrande2018general}) which lack the syntactic
    expressivity to identify all sets of models. In our framework, $\C$-indexed
    collections are not expressive enough to specify \emph{combinations of
    valuations}, since each $\Gamma_c$ only says something about the valuation
    for $c$.
}

Elementariness is characterised by a certain closure condition. Say that two
worlds $W, W'$ are \emph{partition-equivalent} if $\Pi^W_i = \Pi^{W'}_i$ for
all sources $i$, and say $W$ is a \emph{valuation combination} from a set $S
\subseteq \W$ if for all cases $c$ there is $W_c \in S$ such that $v^W_c =
v^{W_c}_c$.  Then a set is elementary iff it is closed under valuation
combinations of partition-equivalent worlds.

\begin{proposition}
\label{prop:elementary_characterisation}
    $S \subseteq \W$ is elementary if and only if the following condition
    holds: for all $W \in \W$ and $W_1, W_2 \in S$, if $W$ is
    partition-equivalent to both $W_1, W_2$ and $W$ is a valuation combination
    from $\{W_1, W_2\}$, then $W \in S$.
\end{proposition}

\section{The Problem}
\label{sec:the_problem}

With the framework set out, we can formally define the problem. We
seek an operator with the following behaviour:

\begin{itemize}

\item {\bf Input:} A sequence of reports $\sigma$, where each report is a
      triple $\tuple{i, c, \phi} \in \S \times \C \times \lprop$ and $\phi
      \not\equiv \falsum$. Such a report represents that {\em source $i$
      reports $\phi$ to hold in case $c$}. Note that we only allow sources to
      make \emph{propositional} reports.

\item {\bf Output:} A pair $\tuple{B^\sigma, K^\sigma}$, where $B^\sigma =
      \{B^\sigma_c\}_{c \in \C}$ is a collection of \emph{belief sets}
      $B^\sigma_c \subseteq \lext$ and $K^\sigma = \{K^\sigma_c\}_{c \in \C}$
      is a collection of \emph{knowledge sets} $K^\sigma_c \subseteq \lext$.

\end{itemize}

\subsection{Basic Postulates}
\label{sec:basic_postulates}

We immediately narrow the scope of operators under consideration by introducing
some basic postulates which are expected to hold. In what follows, say a
sequence $\sigma$ is \emph{$\ast$-consistent} if for each $c \in \C$ the set
$\{\phi \mid \tuple{\ast, c, \phi} \in \sigma\} \subseteq \lprop$ is
classically consistent. Write $G^\sigma_\snd$ for the collection with
$(G^\sigma_\snd)_c = \{S_i(\phi) \mid \tuple{i, c, \phi} \in \sigma\}$, i.e.
the collection of soundness statements corresponding to the reports in
$\sigma$.

\begin{postulate}[\closure{}]
    $B^\sigma = \cn(B^\sigma)$ and  $K^\sigma = \cn(K^\sigma)$
\end{postulate}
\begin{postulate}[\containment{}]
    $K^\sigma \sqsubseteq B^\sigma$
\end{postulate}
\begin{postulate}[\consistency{}]
    If $\sigma$ is $\ast$-consistent, $B^\sigma$ and $K^\sigma$ are
    consistent
\end{postulate}
\begin{postulate}[\soundness{}]
    If $\tuple{i, c, \phi} \in \sigma$, then $S_i(\phi) \in K^\sigma_c$
\end{postulate}
\begin{postulate}[\kbound{}]
    $K^\sigma \sqsubseteq \cn(G^\sigma_\snd \sqcup K^\emptyset)$
\end{postulate}
\begin{postulate}[\priorext{}]
    $K^\emptyset \sqsubseteq K^\sigma$
\end{postulate}
\begin{postulate}[\rearr{}]
    If $\sigma$ is a permutation of $\rho$, then $B^\sigma = B^\rho$ and
    $K^\sigma = K^\rho$
\end{postulate}
\begin{postulate}[\equivpost{}]
    If $\phi \equiv \psi$ then $B^{\sigma \concat \tuple{i, c, \phi}} = B^{\sigma
    \concat \tuple{i, c, \psi}}$ and $K^{\sigma \concat \tuple{i, c, \phi}} =
    K^{\sigma \concat \tuple{i, c, \psi}}$
\end{postulate}

\closure{} says that the belief and knowledge collections are closed under
logical consequence. In light of earlier remarks,
this implies that the propositional belief sets $\proppart{B^\sigma_c}$ are
closed under (propositional) consequence, and that $E_i(\phi) \in B^\sigma_c$
iff $E_i(\phi) \in B^\sigma_d$.
\containment{} says that everything which is known is also believed.
\consistency{} ensures the output is always consistent, provided
we are not in the degenerate case where $\ast$ gives inconsistent reports.
\soundness{} says we \emph{know} that all reports are sound in their respective
cases. This formalises our assumption that sources are \emph{honest}, i.e. that
false reports only arise due to lack of expertise. By \cref{prop:validities}
part (\labelcref{item:e_and_s_implies_phi}) it also implies \emph{experts are
always right}: if a source has expertise on their report then it must be true.
While \soundness{} places a lower bound on knowledge, \kbound{} places an upper
bound: knowledge cannot go beyond the soundness statements corresponding to the
reports in $\sigma$ together with the prior knowledge $K^\emptyset$. That is,
from the point view of knowledge, a new report of $\tuple{i, c, \phi}$ only
allows us to learn $S_i(\phi)$ in case $c$ (and to combine this with other
reports and prior knowledge). Note that the analogous property for belief is
\emph{not} desirable: we want to be more liberal when it comes to beliefs, and
allow for \emph{defeasible inferences} going beyond the mere fact that reports
are sound.
\priorext{} says that knowledge after a sequence $\sigma$ extends the prior
knowledge on the empty sequence $\emptyset$.
\rearr{} says that the order in which reports are received is irrelevant. This
can be justified on the basis that we are reasoning about \emph{static worlds}
for each case $c$, so that there is no reason
to see more ``recent'' reports as any more or less important or truthful than
earlier ones.\footnote{This argument is from \cite{delgrande2006iterated}.}
Consequently, we can essentially view the input as a \emph{multi-set} of belief
sets -- one for each source -- bringing us close to the setting of belief
merging. This postulate also appears as the commutativity postulate
\textbf{(Com)} in the work of \citet{schwind2020}.
Finally, \equivpost{} says that the syntactic form of reports is irrelevant.

Taking all the basic postulates together, the knowledge component $K^\sigma$ is
fully determined once $K^\emptyset$ is chosen.

\begin{proposition}
    \label{prop:prior_knowledge}
    Suppose an operator satisfies the basic postulates. Then
    \begin{enumerate}
        \item $K^\sigma = \cn(G^\sigma_\snd \sqcup K^\emptyset)$
        \item $K^\emptyset = \cn(\emptyset)$ iff $K^\sigma =
              \cn(G^\sigma_\snd)$ for all $\sigma$.
    \end{enumerate}
\end{proposition}

The choice of $K^\emptyset$ depends on the scenario one wishes to model.
While $\cn(\emptyset)$ is a sensible choice if the sequence $\sigma$ is all we
have to go on, we allow $K^\emptyset \ne \cn(\emptyset)$ in case \emph{prior
knowledge} is available (for example, the expertise of particular sources may
be known ahead of time).

Another important property of knowledge, which follows from the basic
postulates, says that \emph{knowledge is monotonic}: knowledge after receiving
$\sigma$ and $\rho$ together is just the case-wise union of $K^\sigma$ and
$K^\rho$.

\begin{postulate}[\kconj{}]
    $K^{\sigma \concat \rho} = \cn(K^\sigma \sqcup K^\rho)$
\end{postulate}

\kconj{} reflects the idea that one should be cautious when it comes to
knowledge: a formula should only be accepted as known if it won't be given up
in light of new information.

\begin{proposition}
    \label{prop:kconj}
    Any operator satisfying the basic postulates satisfies \kconj{}.
\end{proposition}

The postulates also imply some useful properties linking \emph{trust} (seen as
belief in expertise) and \emph{belief/knowledge}.

\begin{proposition}
    \label{prop:basic_postulates_consequences}
    Suppose an operator satisfies the basic postulates. Then
    \begin{enumerate}
        \item \label{item:knowledge_trust_link} If $\phi \in K^\sigma_c$ and
              $\neg\psi \in \cnprop(\phi)$ then $\neg E_i(\psi) \in K^{\sigma
              \concat \tuple{i, c, \psi}}_c$.
        \item \label{item:trust_belief_link} If $\tuple{i, c, \phi} \in
              \sigma$ and $E_i(\phi) \in B^\sigma_c$ then $\phi \in
              B^\sigma_c$.
    \end{enumerate}
\end{proposition}

(\labelcref{item:knowledge_trust_link}) expresses how knowledge can negatively
affect trust: we should distrust sources who make reports we know to be false.
(\labelcref{item:trust_belief_link}) expresses how trust affects belief: we
should believe reports from trusted sources.
It can also be seen as a form of \emph{success} for ordinary sources, and
implies AGM success when $i = \ast$ (by \cref{prop:validities} part
(\labelcref{item:star_exp}) and \closure{}). We illustrate the basic postulates
by formalising the introductory hospital example.

\begin{example}
    \label{ex:hospital_ex_formalised}
    Set $\S, \C$ and $\propvars$ as in \cref{ex:hospital_world}, and consider
    the sequence
    \[
        \sigma
        = (
            \tuple{\ast, c_1, x},
            \tuple{a, c_2, x},
            \tuple{b, c_2, \neg x},
            \tuple{a, c_1, \neg x}
        ).
    \]
    What do we know on the basis of this sequence, assuming the basic
    postulates? First note that by \soundness{}, \cref{prop:validities}
    part (\labelcref{item:star_exp}) and \closure{}, the report from $\ast$ gives $x
    \in K^\sigma_{c_1}$, i.e. reliable reports are known. \soundness{} also
    gives $S_a(x) \land S_b(\neg x) \in K^\sigma_{c_2}$. Combined with
    \cref{prop:validities} parts (\labelcref{item:e_symmetric}),
    (\labelcref{item:e_and_s_implies_phi}) and \closure{}, this yields
    $\neg(E_a(x) \land E_b(x)) \in K^\sigma_c$ for all $c$, formalising the
    intuitive idea that Drs.\ A and B cannot both be experts on X, since they
    give conflicting reports.
    Considering the final report from $a$, we get $x \land S_a(\neg
    x) \in K^\sigma_{c_1}$, and thus $\neg E_a(x) \in K^\sigma_c$ by
    \closure{}. So in fact Dr.\ A is known to be a non-expert on X.

    What about beliefs? The basic postulates do not require beliefs to go
    beyond knowledge, so we cannot say much in general. An ``optimistic"
    operator, however, may opt to believe that sources are experts unless we
    know otherwise, and thus maximise the information that can be (defeasibly)
    inferred from the sequence (in the next section we will introduce
    concrete operators obeying this principle). In this case we may believe
    that at least one source has expertise on $x$ (i.e. $E_a(x) \lor E_b(x) \in
    B^\sigma_c$).  Combined with $\neg E_a(x) \in K^\sigma_c$, \closure{} and
    \containment{}, we get $E_b(x) \in B^\sigma_{c_2}$. Symmetry of expertise
    together with \cref{prop:basic_postulates_consequences} part
    (\labelcref{item:trust_belief_link}) then gives $\neg x \in
    B^\sigma_{c_2}$, i.e. we trust Dr.\ B in the example and believe patient 2
    does not suffer from condition X.

\end{example}

\subsection{Model-based Operators}

While an operator is a purely syntactic object, it will be convenient to
specify $K^\sigma$ and $B^\sigma$ in semantic terms by selecting a set of
\emph{possible} and \emph{most plausible} worlds for each sequence $\sigma$.
We call such operators \emph{model-based}.

\begin{definition}
\label{def:model_based}
An operator is \emph{model-based} if for every $\sigma$ there
are sets $\X_\sigma, \Y_\sigma \subseteq \W$ such that
\begin{inlinelist}
    \item $\X_\sigma \supseteq \Y_\sigma$;
    \item $\Phi \in K^\sigma_c$ iff $W, c \models \Phi$ for all $W \in
          \X_\sigma$; and
    \item $\Phi \in B^\sigma_c$ iff $W, c \models \Phi$ for all $W \in
          \Y_\sigma$.
\end{inlinelist}
\end{definition}

In other words, $K^\sigma_c$ (resp., $B^\sigma_c$) contains the formulas which
hold at case $c$ in \emph{all worlds} in $\X_\sigma$ (resp., $\Y_\sigma)$. It
follows from the relevant definitions that $\X_\sigma \subseteq
\mods(K^\sigma)$, and equality holds if and only if $\X_\sigma$ is elementary
(similarly for $\Y^\sigma$ and $B^\sigma$).
Model-based operators are characterised by our first two basic postulates.

\begin{theorem}
\label{thm:model_based_characterisation}
An operator satisfies \closure{} and \containment{} if and only if it
is model-based.
\end{theorem}

Since we take \closure{} and \containment{} to be fundamental properties, all
operators we consider from now on will be model-based.
We introduce our first concrete operator.

\begin{definition}
    \label{def:weakop}
    Define the model-based operator \weakop{} by
    \[
        \X_\sigma = \Y_\sigma = \{
            W \mid W, c \models S_i(\phi) \text{ for all } \tuple{i, c, \phi}
            \in \sigma
        \}.
    \]
\end{definition}

That is, the possible worlds $\X_\sigma$ are
exactly those satisfying the soundness constraint for each report in
$\sigma$, i.e. false reports are only due to lack of expertise of the
corresponding source.
Syntactically, $K^\sigma = B^\sigma = \cn(G^\sigma_\snd)$.

Clearly \weakop{} satisfies \soundness{}, and one
can show that it satisfies all of the basic postulates of
\cref{sec:basic_postulates}.\footnotemark In fact, it is the \emph{weakest}
operator satisfying \closure{}, \containment{} and \soundness{}, in that
for any other operator $\sigma \mapsto \tuple{\hat{B}^\sigma,
\hat{K}^\sigma}$ with these properties we have $B^\sigma \sqsubseteq
\hat{B}^\sigma$ and $K^\sigma \sqsubseteq \hat{K}^\sigma$ for any $\sigma$.
\footnotetext{
    For \consistency{}, note that for any $\ast$-consistent sequence
    $\sigma$ one can form a world $W$ such that $v_c$ is a model of all
    reports from $\ast$ at case $c$, and $\Pi_i = \{\vals\}$ for all $i \ne
    \ast$. This satisfies all the soundness constraints, so $W \in
    \X_\sigma = \Y_\sigma$.
}

\begin{example}
\label{ex:model_based}
    Consider \weakop{} applied to the sequence
    $
        \sigma
        = (\tuple{\ast, c, p}, \tuple{i, c, \neg p \land q})
    $.
    By \soundness{}, \closure{} and the validities from \cref{prop:validities},
    we have $p \in K^\sigma_c$ and $\neg E_i(p) \in K^\sigma_c$. In fact, by
    \closure{}, we
    have $\neg E_i(p) \in K^\sigma_d$ for all cases $d$.
    However, we cannot say much about $q$: neither $q$, $\neg q$, $E_i(q)$ nor
    $\neg E_i(q)$ are in $B^\sigma_c = K^\sigma_c$.

\end{example}

\section{Constructions}
\label{sec:constructions}

For model-based operators in \Cref{def:model_based}, the sets $\X_\sigma$ and
$\Y_\sigma$ -- which determine knowledge and belief -- can
depend on $\sigma$ in a completely arbitrary manner. This lack of structure
leads to very wide class of operators, and one cannot say much about them in
general beyond the satisfaction of \closure{} and \containment{}. In this
section we specialise model-based operators by providing two constructions.

\subsection{Conditioning Operators}
\label{sec:conditioning_operators}

Intuitively, $\Y_\sigma$ is supposed to represent the \emph{most
plausible} worlds among the possible worlds in $\X_\sigma$. This suggests the
presence of a \emph{plausibility ordering} on $\X_\sigma$, which is used to
select $\Y_\sigma$.
For our first construction we take this approach: we condition a fixed
plausibility total preorder\footnotemark{} on the knowledge $\X_\sigma$, and
obtain $\Y_\sigma$ by selecting the minimal (i.e. most plausible) worlds.
\footnotetext{
    A total preorder is a reflexive, transitive and total relation.
}

\begin{definition}
\label{def:conditioning_operator}
An operator is a \emph{conditioning operator} if there is a total
preorder $\le$ on $\W$ and a mapping $\sigma \mapsto \tuple{\X_\sigma,
\Y_\sigma}$ as in \cref{def:model_based} such that
$
    \Y_\sigma = \min_{\le}{\X_\sigma}
$
for all $\sigma$.
\end{definition}

Note that $\le$ is independent of $\sigma$: it is fixed before receiving any
reports. All conditioning operators are model-based by definition. Clearly
$\Y_\sigma$ is determined by $\X_\sigma$ and the plausibility order, so that to
define a conditioning operator it is enough to specify $\le$ and the mapping
$\sigma \mapsto \X_\sigma$.
Write $W \simeq W'$ iff both $W \le W'$ and $W' \le W$.
We now present examples of how such an ordering can be defined.

\begin{definition}
    \label{def:varbasedcond}
    Define the conditioning operator \varbasedcond{} by setting
    $\X_\sigma$ in the same way as \weakop{} in \cref{def:weakop}, and $W \le
    W'$ iff $r(W) \le r(W')$, where
    \[
        r(W) = - \sum_{i \in \S}\left|\left\{
            p \in \propvars
            \mid
            \Pi^W_i[p] = \propmods(p)
        \right\}\right|.
    \]
\end{definition}

\varbasedcond{} aims to trust each source on \emph{as many propositional
variables} as possible. One can check that \varbasedcond{} satisfies the
basic postulates.

\begin{example}
    \label{ex:conditioning_operator}
    Revisiting the sequence
    $
        \sigma
        = (\tuple{\ast, c, p}, \tuple{i, c, \neg p \land q})
    $
    from \cref{ex:model_based} with \varbasedcond{}, the knowledge set
    $K^\sigma_c$ is the same as before, but we now have $q \land E_i(q) \in
    B^\sigma_c$. This reflects the ``credulous'' behaviour of the ranking
    $\le$: while it is not possible to believe $i$ is an expert on $p$, we
    should believe they \emph{are} an expert on $q$ so long as this does not
    conflict with soundness. For the propositional beliefs generally, we have
    $\proppart{B^\sigma_c} = \cnprop(p \land q)$. That is, \varbasedcond{} takes
    the $q$ part of the report from $i$ (on which $i$ is credulously trusted)
    while ignoring the $\neg p$ part (which is false due to report from
    $\ast$).

\end{example}

\begin{definition}
    Define a conditioning operator \partbasedcond{} with $\X_\sigma$ as
    for \varbasedcond{}, and $\le$ defined by the ranking function
    \[
        r(W) = -\sum_{i \in \S}{|\Pi^W_i|}.
    \]
\end{definition}

\partbasedcond{} aims to maximise the \emph{number of cells} in the sources'
partitions, and thereby maximise the number of propositions on which they have
expertise. Unlike \varbasedcond{}, the propositional variables play no special
role. As expected, \partbasedcond{} satisfies the basic postulates.

\begin{example}
    Applying \partbasedcond{} to $\sigma$ from
    \cref{ex:model_based,ex:conditioning_operator}, we no longer extract $q$
    from the report of $i$: $q \notin B^\sigma_c$ and $E_i(q) \notin
    B^\sigma_c$. Instead, we have $\proppart{B^\sigma_c} = \cnprop(p)$, and
    $E_i(p \lor q) \in B^\sigma_c$.
\end{example}

Note that both \varbasedcond{} and \partbasedcond{} are based on the general
principle of maximising the expertise of sources, subject to the constraint
that all reports are sound. This intuition is formalised by the following
postulate for conditioning operators. In what follows, write $W \preceq W'$ iff
$\Pi^W_i$ refines $\Pi^{W'}_i$ for all $i \in \S$, i.e. if all sources have
broadly more expertise in $W$ than in $W'$.\footnotemark{}

\footnotetext{
    $\Pi$ refines $\Pi'$ if $\forall A \in \Pi$, $\exists B \in \Pi'$ such that
    $A \subseteq B$.
}

\begin{postulate}[\refinement{}]
    If $W \preceq W'$ then $W \le W'$
\end{postulate}

Since $\preceq$ is only a partial order on $\W$ there are many possible total
extensions; \varbasedcond{} and \partbasedcond{} provide two specific examples.

We now turn to an axiomatic characterisation of conditioning operators.
Taken with the basic postulates from \cref{sec:basic_postulates},
conditioning operators can be characterised using an approach similar to that of
\citet{delgrande2018general} in their account of \emph{generalised AGM belief
revision}.\footnotemark{} This involves a technical property
\citeauthor{delgrande2018general}
call \postulatename{Acyc}, which finds its roots in the \emph{Loop}
property of \citet{kraus1990nonmonotonic}.
\footnotetext{ Note that while the result is similar, our framework is not an
instance of theirs.  }

\begin{postulate}[\duprem{}]
    If $\rho_1 = \sigma \concat \tuple{i, c, \phi}$ and $\rho_2 = \rho_1
    \concat \tuple{i, c, \phi}$ then $B^{\rho_1} = B^{\rho_2}$ and
    $K^{\rho_1} = K^{\rho_2}$
\end{postulate}
\begin{postulate}[\condcons{}]
    If $K^\sigma$ is consistent then so is $B^\sigma$
\end{postulate}
\begin{postulate}[\incvac{}]
    $B^{\sigma \concat \rho} \sqsubseteq \cn(B^\sigma \sqcup K^\rho)$, with
    equality if $B^\sigma \sqcup K^\rho$ is consistent
\end{postulate}
\begin{postulate}[\acyc{}]
    If $\sigma_0, \ldots, \sigma_n$ are such that $K^{\sigma_j} \sqcup
    B^{\sigma_{j+1}}$ is consistent for all $0 \le j < n$ and $K^{\sigma_n}
    \sqcup B^{\sigma_0}$ is consistent, then $K^{\sigma_0} \sqcup B^{\sigma_n}$
    is consistent
\end{postulate}

\incvac{} is so-named since it is analogous to the combination of
\emph{Inclusion} and \emph{Vacuity} from AGM revision, if one informally views
$B^{\sigma \concat \rho}$ as the revision of $B^\sigma$ by $K^\rho$.
\condcons{} is another consistency postulate, which follows from
\consistency{},  \closure{} and \soundness{}.
\acyc{} is the analogue of the postulate of \citeauthor{delgrande2018general},
which rules out cycles in the plausibility order
constructed in the representation result.

As with the result of \citeauthor{delgrande2018general}, a technical condition beyond
\cref{def:conditioning_operator} is required to obtain the characterisation:
say that a conditioning operator is \emph{elementary} if for each $\sigma$ the
sets of worlds $\X_\sigma$ and $\Y_\sigma = \min_{\le}\X_\sigma$ are
elementary.\footnotemark{}
\footnotetext{
    Equivalently, there is a total preorder $\le$ such that $\mods(B^\sigma) =
    \min_{\le}\mods(K^\sigma)$ for all $\sigma$.
}

\begin{theorem}
    \label{thm:conditioning_characterisation}
    Suppose an operator satisfies the basic postulates of
    \cref{sec:basic_postulates}.\footnotemark{} Then it is an elementary
    conditioning operator if and only if it satisfies \duprem{},
    \condcons{}, \incvac{} and \acyc{}.

    \footnotetext{
        Strictly speaking, we only need \closure{}, \containment{},
        \kconj{}, \equivpost{} and \rearr{}.
    }
\end{theorem}

The proof roughly follows the lines of Theorem 4.9 in
\cite{delgrande2018general}, although some
differences arise due to the form of our input as finite sequences of reports.
We note that while the requirement that $\X_{\sigma}$ and $\Y_{\sigma}$ are elementary is
a technical condition,\footnotemark{} the characterisation in
\cref{prop:elementary_characterisation} implies a simple sufficient condition
for elementariness.
\footnotetext{
    \incvac{} may fail for non-elementary conditioning.
}

\begin{proposition}
    \label{prop:partition_equiv_tpo_implies_elementariness}
    Suppose $\le$ is such that $W \simeq W'$ whenever $W$ and $W'$ are
    partition-equivalent. Then $\min_{\le}{S}$ is elementary for any
    elementary set $S \subseteq \W$.
\end{proposition}

\Cref{prop:partition_equiv_tpo_implies_elementariness} implies that
\varbasedcond{} and \partbasedcond{} are elementary. Indeed, for both operators
$\X_\sigma = \mods(G^\sigma_\snd)$ so is elementary by definition. Since the
ranking $\le$ for each operator only depends on the partitions of worlds,
$\Y_\sigma =
\min_{\le}{\X_\sigma}$ is elementary also.

\subsection{Score-based operators}
\label{sec:score_based}

The fact that the plausibility order $\le$ of a conditioning operator is fixed
may be too limiting. For example, consider
\[
    \sigma
    =
    (\tuple{i, c, p},
    \tuple{j, c, \neg p},
    \tuple{i, d, p}).
\]
If one sets $\X_\sigma$ to satisfy the soundness constraints (i.e. as in
\weakop{}), there is a possible
world $W_1 \in \X_\sigma$ with $W_1, d \models \neg E_i(p) \land E_j(p) \land
\neg p$ (i.e. $W_1$ sides with source $j$ and $p$ is false at $d$) and another
world $W_2 \in \X_\sigma$ with $W_2, d \models E_i(p) \land \neg E_j(p) \land
p$ (i.e. $W_2$ sides with source $i$). Appealing to symmetry, one may argue
that neither world is \emph{a priori} more plausible than the other, so any
fixed plausibility order should have $W_1 \simeq W_2$. If these worlds
are maximally plausible (e.g. if taking the ``optimistic" view outlined in
\cref{ex:hospital_ex_formalised}), conditioning gives $p \notin B^\sigma_d$ and
$\neg p \notin B^\sigma_d$.
However, there is an argument that $W_2$ should be considered more plausible
than $W_1$ \emph{given the sequence $\sigma$}, since $W_2$ validates the final
report $\tuple{i, d, p}$ whereas $W_1$ does not. Consequently, there is an
argument that we should in fact have $p \in B^\sigma_d$.\footnote{At the very
least, the case $p \in B^\sigma_d$ should not be \emph{excluded}.} This shows
that we need the plausibility order to be responsive to the input sequence for
adequate belief change.\footnotemark{}

\footnotetext{
    In \cref{sec:one_step_postulates} we make this argument more precise
    by providing an impossibility result which shows conditioning operators
    with some basic properties cannot accept $p$ in sequences such as this.
}

As a result of this discussion, we look for operators whose plausibility
ordering can depend on $\sigma$. One approach to achieve this in a controlled
way is to have a ranking for each
\emph{report} $\tuple{i, c, \phi}$, and combine these to construct a ranking
for each sequence $\sigma$. We represent these rankings by \emph{scoring
functions}, and call the resulting operators \emph{score-based}.

\begin{definition}
\label{def:score_based}
    An operator is \emph{score-based} if there is a mapping $\sigma \mapsto
    \tuple{\X_\sigma, \Y_\sigma}$ as in \cref{def:model_based} and functions
    $r_0: \W \to \Ninf$, $d: \W \times (\S \times \C \times \lprop) \to \Ninf$
    such that $\X_\sigma = \{W \mid r_\sigma(W) < \infty\}$ and $\Y_\sigma =
    \argmin_{W \in \X_\sigma}{r_\sigma(W)}$, where
    \[
        r_\sigma(W) = r_0(W) + \sum\nolimits_{\tuple{i, c, \phi} \in \sigma}{
            d(W, \tuple{i, c, \phi})
        }.
    \]

\end{definition}

Here $r_0(W)$ is the \emph{prior implausibility score} of $W$, and $d(W,
\tuple{i, c, \phi})$ is the \emph{disagreement score} for world $W$ and $\tuple{i,
c, \phi}$. The set of most plausible worlds $\Y_\sigma$ consists of those $W$
which minimise the sum of the prior implausibility and the total
disagreement with $\sigma$. Note that by summing the scores
of each report $\tuple{i, c, \phi}$ with equal weight, we treat each report independently.
Score-based operators generalise elementary conditioning operators with \kconj{}.

\begin{proposition}
\label{prop:kconj_conditioning_implies_score_based}
    Any elementary conditioning operator satisfying \kconj{} is score-based.
\end{proposition}



We now give a concrete example.

\begin{definition}
    \label{def:scorebasedop}
    Define a score-based operator \scorebasedop{} by setting
    $r_0(W) = 0$ and
    \[
        d(W, \tuple{i, c, \phi}) = \begin{cases}
            |\Pi^W_i[\phi] \setminus \propmods(\phi)|,& W, c \models S_i(\phi) \\
            \infty,& \text{ otherwise. }
        \end{cases}
    \]
\end{definition}

The set of possible worlds $\X_\sigma$ is the same as for the earlier
operators.
All worlds are \emph{a priori} equiplausible according to
$r_0$. The disagreement score $d$ is defined as the number of propositional
valuations in the ``excess" of $\Pi^W_i[\phi]$ which are not models of
$\phi$, i.e. the number of $\neg\phi$ valuations which are indistinguishable
from some $\phi$ valuation.
%
The intuition here is that \emph{sources tend to only report formulas on which
they have expertise}. The minimum score 0 is attained exactly when $i$ has
expertise on $\phi$; other worlds are ordered by how much they deviate from
this ideal.

One can verify that \scorebasedop{} satisfies the basic postulates of
\cref{sec:basic_postulates}. It can also be seen that $\X_\sigma$ and
$\Y_\sigma$ are elementary, and \scorebasedop{} fails \duprem{} and \incvac{}.
It follows from
\cref{thm:conditioning_characterisation} that \scorebasedop{} is \emph{not} a
conditioning operator.


\begin{example}
\label{ex:score_based}
     To illustrate the differences between \scorebasedop{} and conditioning,
     consider a more elaborate version of the
     example given at the start of this \lcnamecref{sec:score_based}:
     \[
        \sigma = (
            \tuple{i, c, p \rightarrow q},
            \tuple{j, c, p \rightarrow \neg q},
            \tuple{\ast, c, p},
            \tuple{i, d, p},
            \tuple{i, d, q}
        ).
     \]
     Here the reports of $i$ and $j$ in case $c$ are consistent,
     but inconsistent when taken with
     the reliable information $p$ from $\ast$. Should we believe $q$ or $\neg
     q$? Both our conditioning operators \varbasedcond{} and \partbasedcond{}
     decline to decide, and have $\proppart{B^\sigma_c} = \cnprop(p)$. However,
     since \scorebasedop{} takes into account each report in the
     sequence, the fact that $i$ reports both $p$ and $q$ in case $d$ leads to
     $E_i(p) \land E_i(q) \in B^\sigma_c$. This gives $E_i(p \rightarrow q)
     \in B^\sigma_c$ by \cref{prop:validities} part
     (\labelcref{item:exp_on_all_variables}), so we can make use of the report
     from $i$ in case $c$: we have $\proppart{B^\sigma_c} = \cnprop(p \land
     q)$.
     This example shows that score-based operators can be \emph{more credulous}
     than conditioning operators (e.g. we can believe $E_i(p)$ when
     $i$ reports $p$), and can consequently hold stronger propositional
     beliefs.

\end{example}

\section{One-step Revision}
\label{sec:one_step_postulates}

The postulates of \cref{sec:basic_postulates} only set out very basic
requirements for an operator. In this section we introduce some more demanding
postulates which address how beliefs should change when a sequence $\sigma$ is
extended by a new report $\tuple{i, c, \phi}$.
%
First, we address how propositional beliefs should be affected by reliable
information.

\begin{postulate}[\agm{}]
    For any $\sigma$ and $c \in \C$ there is an AGM operator $\star$
    for $\proppart{B^\sigma_c}$ such that $\proppart{B^{\sigma \concat
    \tuple{\ast, c, \phi}}_c} = \proppart{B^\sigma_c} \star \phi$ whenever
    $\neg\phi \notin K^\sigma_c$
\end{postulate}

\agm{} says that receiving information from the reliable source $\ast$ acts in
accordance with the well-known AGM postulates~\cite{alchourron1985logic} for propositional belief
revision (provided we are not in the degenerate case where the new report
$\phi$ was already \emph{known} to be false). Since AGM revision operators are
characterised by total preorders over valuations
\cite{grove1988two,katsuno_1991}, it is no surprise that our order-based
constructions are consistent with \agm{}.

\begin{proposition}
    \label{prop:examples_satisfy_agm}
    \varbasedcond{}, \partbasedcond{} and \scorebasedop{}
    satisfy \agm{}.
\end{proposition}

Thus, we do indeed extend AGM revision in the case of reliable information.
What about non-reliable information? First note that the analogue of \agm{} for
ordinary sources $i \ne \ast$ is \emph{not} desirable. In particular, we
should not have the \postulatename{Success} postulate:
\[
    \phi \in B^{\sigma \concat \tuple{i, c, \phi}}_c.
\]
Indeed, the sequence in \cref{ex:model_based} with $\phi = \neg p
\land q$ already shows that \postulatename{Success} would conflict with the
basic postulates.
However, there are weaker modifications of \postulatename{Success} which may be
more appropriate. We consider two such postulates.

\begin{postulate}[\condsucc{}]
    If $E_i(\phi) \in B^\sigma_c$ and $\neg\phi \notin B^\sigma_c$, then $\phi
    \in B^{\sigma \concat \tuple{i, c, \phi}}_c$
\end{postulate}

\begin{postulate}[\strongcondsucc{}]
    If $\neg(E_i(\phi) \land \phi) \notin B^\sigma_c$, then $\phi \in B^{\sigma
    \concat \tuple{i, c, \phi}}_c$
\end{postulate}

\condsucc{} says that if $i$ is deemed an expert on $\phi$, which is consistent
with current beliefs, then $\phi$ is accepted after $i$ reports it. That is,
the acceptance of $\phi$ is \emph{conditional} on prior beliefs about the expertise of
$i$ (on $\phi$). \strongcondsucc{} weakens the antecedent by only requiring
that $E_i(\phi)$ and $\phi$ are jointly consistent with current beliefs (i.e.
$i$ need not be considered an expert on $\phi$). In other words, we should
believe reports if there is no reason not to. It is easily shown that
\closure{} and \strongcondsucc{} implies \condsucc{}.
We once again revisit our examples.

\begin{proposition}
    \label{prop:examples_satisfy_condsucc}
    \varbasedcond{}, \partbasedcond{} and \scorebasedop{}
    satisfy \condsucc{}, and \scorebasedop{} additionally
    satisfies \strongcondsucc{}.
\end{proposition}

By omission, the reader may suppose that the conditioning operators fail
\strongcondsucc{}. This is correct, and we can in fact say even more: \emph{no}
conditioning operator with a few basic properties -- all of which are satisfied
by \varbasedcond{} and \partbasedcond{} -- can satisfy \strongcondsucc{}.
In what follows, for a permutation $\pi: \S \to \S$ with $\pi(\ast) = \ast$,
write $\pi(W)$ for the world with $v^{\pi(W)}_c = v^W_c$ and $\Pi^{\pi(W)}_i =
\Pi^W_{\pi(i)}$. We have an impossibility result.

\begin{proposition}
    \label{prop:strongcondsucc_conditioning_impossibilitity}
    No elementary conditioning operator satisfying the basic postulates can
    simultaneously satisfy the following properties:
    \begin{enumerate}
        \item \label{item:conditioning_impossibility_first}
              $K^\emptyset = \cn(\emptyset)$

        \item \label{item:anonymity}
              If $\pi$ is a permutation of $\S$ with $\pi(\ast) = \ast$, $W
              \simeq \pi(W)$

        \item \refinement{}

        \item \label{item:conditioning_impossibility_last} \strongcondsucc{}
    \end{enumerate}
\end{proposition}

(\labelcref{item:conditioning_impossibility_first}) says that before any
reports are received, we only know tautologies. As remarked earlier, this is
not an \emph{essential} property, but is reasonable when no prior knowledge is
available. (\labelcref{item:anonymity}) is an anonymity postulate: it says that
permuting the ``names" of sources does not affect the plausibility of a world,
and is a desirable property in light of
(\labelcref{item:conditioning_impossibility_first}). \refinement{}, introduced
in \cref{sec:conditioning_operators}, says that worlds in which all sources
have more expertise are preferred.

\Cref{prop:strongcondsucc_conditioning_impossibilitity} highlights an important
difference between conditioning and score-based operators, and hints that
a fixed plausibility order may be too restrictive: we
need to allow the order to be responsive to new reports in order to satisfy
properties such as \strongcondsucc{}.

\section{Selective Change}
\label{sec:selective_change}

In the previous section we saw how a single formula $\phi$ may be accepted when
it is received as an additional report. But what can we say about propositional
beliefs when taking into account the \emph{whole sequence} $\sigma$? To
investigate this we introduce an analogue of \emph{selective revision}
\cite{ferme1999selective}, in which propositional beliefs are
formed by ``selecting" part of each input report (intuitively, some part
consistent with the source's expertise). In what follows, write $\sigma \rs c =
\{\tuple{i, \phi} \mid \tuple{i, c, \phi} \in \sigma\}$ for the $c$-reports in
$\sigma$.

\begin{definition}
    \label{def:selectivity}
    A \emph{selection scheme} is a mapping $f$ assigning to each
    $\ast$-consistent sequence $\sigma$ a function $f_\sigma: \S \times \C
    \times \lprop \to \lprop$ such that $f_\sigma(i, c, \phi) \in
    \cnprop(\phi)$.
    An operator is \emph{selective} if there is a selection scheme $f$ such
    that for all $\ast$-consistent $\sigma$ and $c \in \C$,
    \[
        \proppart{B^\sigma_c} = \cnprop(\{f_\sigma(i, c, \phi) \mid \tuple{i,
        \phi} \in \sigma \rs c\}).
    \]
\end{definition}

Thus, an operator is selective if its propositional beliefs in case $c$ are
formed by weakening each $c$-report and taking their consequences. Note that for
$\sigma = \emptyset$ we get $\proppart{B^\sigma_c} = \cnprop(\emptyset)$, so
selectivity already rules out non-tautological prior propositional beliefs.
Also note that in the presence of \closure{}, \containment{} and \soundness{},
selectivity implies that $\proppart{B^\sigma_c} = \proppart{B^\rho_c}$, where
$\rho$ is obtained by replacing each report $\tuple{i, c, \phi}$ with
$\tuple{\ast, c, f_\sigma(i, c, \phi)}$.

Selectivity can be characterised by a natural postulate placing an upper bound
on the propositional part of $B^\sigma_c$. In what follows, let
$\Gamma^\sigma_c = \{\phi \in \lprop \mid \exists i \in \S: \tuple{i, \phi}
\in \sigma \rs c\}$.

\begin{postulate}[\boundedness{}]
    If $\sigma$ is $\ast$-consistent, $\proppart{B^\sigma_c}
    \subseteq \cnprop(\Gamma^\sigma_c)$
\end{postulate}

\boundedness{} says that the propositional beliefs in case $c$ should not go
beyond the consequences of the formulas reported in case $c$. In some sense
this can be seen as an iterated version of \postulatename{Inclusion} from AGM
revision, in the case where $\proppart{B^\sigma_c} = \cnprop(\emptyset)$. We
have the following characterisation.

\begin{theorem}
    \label{thm:selectivity_characterisation}
    A model-based operator is selective if and only if it satisfies
    \boundedness{}.
\end{theorem}


This result allows us to easily analyse when conditioning and score-based
operators are selective; we show in the
\ifdefined\thisistheprerint
    appendix
\else
    full paper~\cite{singleton_booth_22_preprint}
\fi
that \varbasedcond{}, \partbasedcond{} and \scorebasedop{} are all selective.

\subsection{Case independence}

In the definition of a selection scheme, we allow $f_\sigma(i, c, \phi)$ to
depend on the case $c$. If one views $f_\sigma(i, c, \phi)$ as a weakening of
$\phi$ which accounts for the lack of expertise of $i$, this is somewhat at
odds with other aspects of the framework, where expertise is independent of
case. For this reason it is natural to consider \emph{case independent}
selective schemes.

\begin{definition}
    \label{def:case_independent_selectivity}
    A selection scheme $f$ is \emph{case independent} if $f_\sigma(i, c, \phi)
    \equiv f_\sigma(i, d, \phi)$ for all $\ast$-consistent $\sigma$ and $i \in
    \S$, $c, d \in \C$ and $\phi \in \lprop$.
\end{definition}

Say an operator is \emph{case-independent-selective} if it is selective
according to some case independent scheme. This stronger notion of selectivity
can again be characterised by a postulate which bounds propositional beliefs.
For any set of cases $H \subseteq \C$, sequence $\sigma$ and $c \in \C$, write
\begin{align*}
    \Gamma^{\sigma,H}_c
    = \{
        \phi \in \lprop
        \mid
        \exists i \in \S:
            &\tuple{i, \phi} \in \sigma \rs c \\
            &\text{ and }
            \forall d \in H : \tuple{i, \phi} \notin \sigma \rs d
    \}.
\end{align*}

\begin{postulate}[\hboundedness{}]
    For any $\ast$-consistent $\sigma$, $H \subseteq \C$ and $c \in \C$,
    \[
        \proppart{B^\sigma_c}
            \subseteq
            \cnprop\left(
                \Gamma^{\sigma,H}_c
                \cup
                \bigcup_{d \in H}{
                    \proppart{B^\sigma_d}
                }
            \right)
    \]
\end{postulate}

Note that \boundedness{} is obtained as the special case where $H =
\emptyset$. We illustrate with an example.

\begin{example}
    \label{ex:hprop}
    Consider case $c$ in the following sequence:
    \[
        \sigma
        = (
            \tuple{i, c, p},
            \tuple{j, c, q},
            \tuple{j, d, q},
            \tuple{k, d, r}
        )
    \]
    \boundedness{} requires that $\proppart{B^\sigma_c} \subseteq \cnprop(\{p,
    q\})$. However, the instance of \hboundedness{} with $H = \{d\}$ makes use
    of the fact that $j$ reports $q$ in both cases $c$ and $d$, and requires
    $\proppart{B^\sigma_c} \subseteq \cnprop(\{p\} \cup
    \proppart{B^\sigma_d})$.
    This also has an interesting implication for case $d$: if $\phi \in
    \proppart{B^\sigma_c}$, then $p \rightarrow \phi \in
    \proppart{B^\sigma_d}$. This follows since $\beta \in \cnprop(\{\alpha\}
    \cup \Gamma)$ iff $\alpha \rightarrow \beta \in \cnprop(\Gamma)$ for
    $\alpha, \beta \in \lprop$.
    Intuitively, this says that if $p$ (from $i$) and $q$
    (from $j$) is enough to accept $\phi$ in case $c$, then $\phi$ is
    accepted in case $d$ \emph{if $p$ is}, given that the report of $q$ from
    $j$ is repeated for $d$.

\end{example}


\begin{theorem}
    \label{thm:case_independent_selectivity_characterisation}
    A model-based operator is case-independent-selective if and only if it
    satisfies \hboundedness{}.
\end{theorem}

The question of whether our concrete operators satisfy \hboundedness{}
(equivalently, whether they are case-independent-selective) is still open.

\section{Related Work}
\label{sec:relatedwork}


\paragraph{Belief merging.}

In the framework of \citet{konieczny2002merging}, a merging
operator $\Delta$ maps a multiset of propositional formulas $\Phi =
\{\phi_1,\ldots,\phi_n\}$ and an integrity constraint $\mu$ to a formula
$\Delta_\mu(\Phi)$.
%
%
This can be seen as the special case of our framework with a single
case $c$: for $\Phi$, $\mu$ we consider the sequence
$\sigma_{\Phi, \mu}$ where $\ast$ reports $\mu$ and each source $i$ reports
$\phi_i$.
%
Any operator then gives rise to a merging operator $\Delta_\mu(\Phi) =
\bigwedge\proppart{B^{\sigma_{\Phi,\mu}}_c}$.
%

We go beyond this setting by considering multiple cases and explicitly
modelling expertise (and trust, via beliefs about expertise). While it may be
possible to model expertise \emph{implicitly} in belief merging (for example,
say $i$ is not trusted on $\psi$ if $\Delta_\mu(\Phi) \not\vdash \psi$ when
$\phi_i \vdash \psi$), bringing expertise to the object level allows us to
express more complex beliefs about expertise, such as $E_a(x) \lor E_b(x)$ in
\cref{ex:hospital_ex_formalised}. It also facilitates postulates which refer
directly to expertise, such as the weakenings of \postulatename{Success} in
\cref{sec:one_step_postulates}.

However, our problem is more specialised than merging, since we focus
specifically on conflicting information due to lack of expertise. Belief
merging may be applied more broadly to other types of \emph{information
fusion}, e.g. subjective beliefs or goals~\cite{gregoire_fusion_2006},
where notions of objective expertise do not apply. While our framework
\emph{could} be applied in these settings, our postulates may no longer be
desirable.

\paragraph{Epistemic logic.}

Our notions of expertise and soundness are related to \emph{S5
knowledge} from epistemic logic~\cite{handbook_epistemic}. In such
logics, an agent \emph{knows} $\phi$ at a state $x$ if $\phi$ holds at all
states $y$ ``accessible" from $x$. Knowledge is thus determined by an
\emph{epistemic accessibility relation}, which describes the distinctions
between states the agent can make. The logic of S5 arises when this relation is an
equivalence relation (or equivalently, a partition).

Our previous work~\cite{singleton2021logic} -- in which expertise and soundness
were introduced in a modal logic framework -- showed that ``expertise
models" are in 1-to-1 correspondence with S5 models, such that $E(\phi)$
holds iff $A(\phi \rightarrow K\phi)$ holds in the S5 model, where $A$ is the
universal modality. By symmetry of expertise, we can also replace $\phi$ with
its negation. Thus, expertise has a precise epistemic interpretation: it is the
ability to \emph{know whether} $\phi$ holds in \emph{any possible state}.
Similarly, $S(\phi)$ translates to $\neg K \neg \phi$. That is, $\phi$ is sound
exactly when the source does not \emph{know} $\phi$ is false.

In the present framework, if we set
$W, c \models K_i(\phi)$ iff $\Pi_i[v_c] \subseteq \propmods(\phi)$ and
$W, c \models A\Phi$ iff $\forall v:\ W_{c=v}, c \models \Phi$,
where $W_{c=v}$ is the world obtained from $W$ by setting $v'_c = v$, then we
have $E_i(\phi) \equiv A(\phi \rightarrow K_i(\phi))$ and $S_i(\phi) \equiv
\neg K_i(\neg \phi)$. While $K_i$ is not quite an S5 modality (the \textbf{5}
axiom requires iterating $K_i$, which is not possible in our framework), this
shows the fundamental link between expertise, soundness and knowledge.

\section{Conclusion}
\label{sec:conclusion}

\paragraph{Summary.} In this paper we studied a belief change problem --
extending the classical AGM framework -- in which
beliefs about the state of the world in multiple cases, as well as expertise of
multiple sources, must be inferred from a sequence of reports. This allowed us
to take a fresh look at the
interaction between trust (seen as \emph{belief in expertise}) and belief.
By inferring the expertise of the sources from the reports,
we have generalised some earlier approaches to non-prioritised
revision which assume expertise (or reliability, credibility, priority
etc) is known up-front (e.g.
\cite{ferme1999selective,hansson_2001,booth_trust_2018,delgrande2006iterated}).
We went on to propose some concrete belief change operators, and explored their
properties through examples and postulates.

We saw that conditioning operators satisfy some desirable
properties, and our concrete instances make useful inferences that go
beyond \weakop{}. However, we have examples in which intuitively plausible
inferences are blocked, and conditioning is largely incompatible with
\strongcondsucc{}. Score-based operators, and in particular \scorebasedop{},
offer a way around these limitations, but may
come at the expense of some other seemingly reasonable postulates, such as
\duprem{}.

\paragraph{Future work.} There are many possibilities for
future work.
Firstly, we have a representation result only for conditioning operators. A
characterisation of score-based operators -- either the class in general or the
specific operator \scorebasedop{} -- remains to be found. This would help to
further clarify the differences between conditioning and score-based operators.
We have also not considered any computational issues. Determining the
complexity of calculating the results of our example operators, and the
complexity for conditioning and score-based operators more broadly, is left to
future work.
Secondly, there is scope for deeper postulate-based analysis. For example,
there should be postulates governing how beliefs change in case $c$ in response
to reports in case $d$. We could also
consider more postulates relating trust and belief, and compare these
postulates with those of \citet{yasser_21}.
%
%
Finally, our framework only deals with three levels of trust on a proposition:
we can believe $E_i(\phi)$, believe $\neg E_i(\phi)$, or neither.
Future work could investigate how to extend our semantics to talk about \emph{graded
expertise}, and thereby permit more fine-grained \emph{degrees of trust}
\cite{hunter_building_21,yasser_21,delgrande2006iterated}.

\section*{Acknowledgements}

We thank the anonymous KR 2022 reviewers for their thoughtful comments and
suggestions.

\appendix











\section{Proofs}

\subsection{Proof of \cref{prop:validities}}

\begin{proof}\leavevmode
\begin{itemize}
    \item[\labelcref{item:replacement_equivalents_e_s}.]
        If $\phi \equiv \psi$ then $\propmods(\phi) = \propmods(\psi)$; since
        the semantics for $\S_i(\phi)$ and $E_i(\phi)$ only refer to
        $\propmods(\phi)$ (and likewise for $\psi$), we have that $S_i(\phi)
        \leftrightarrow S_i(\psi)$ and $E_i(\phi) \leftrightarrow E_i(\psi)$
        are valid.

    \item[\labelcref{item:e_symmetric}.]
        For the first validity, suppose $W, c \models E_i(\phi)$. Then
        $\propmods(\phi) = \Pi_i[\phi]$. We show $W, c \models E_i(\neg\phi)$.
        Indeed, take $v \in \Pi_i[\neg\phi]$. Then there is $v' \in
        \propmods(\neg\phi)$ such that $v \in \Pi_i[v']$. Thus $v' \in
        \Pi_i[v]$ also. Supposing for contradiction that $v \in
        \propmods(\phi)$, we get
        \[
            v' \in \Pi_i[v] \subseteq \Pi_i[\phi] = \propmods(\phi).
        \]
        But then $v' \in \propmods(\neg\phi) \cap \propmods(\phi) = \emptyset$;
        contradiction. Hence $v \notin \propmods(\phi)$, i.e. $v \in
        \propmods(\neg\phi)$. This shows that $\Pi_i[\neg\phi] \subseteq
        \propmods(\neg\phi)$, so $W, c \models E_i(\neg\phi)$.

        We have shown that $E_i(\phi) \rightarrow E_i(\neg\phi)$ is valid. For
        the converse note that, by symmetry, $E_i(\neg\phi) \rightarrow
        E_i(\neg\neg\phi)$ is valid; since $E_i(\neg\neg\phi)$ is equivalent to
        $E_i(\phi)$ by (\labelcref{item:replacement_equivalents_e_s}) we get
        $E_i(\phi) \leftrightarrow E_i(\neg\phi)$.

        For the second validity, suppose $W, c \models E_i(\phi) \land
        E_i(\psi)$. Note that
        \[
            \Pi_i[\phi \land \psi]
            \subseteq \Pi_i[\phi]
            = \propmods(\phi)
        \]
        and, similarly, $\Pi_i[\phi \land \psi] \subseteq \propmods(\psi)$.
        Hence
        \[
            \Pi_i[\phi \land \psi]
            \subseteq \propmods(\phi) \cap \propmods(\psi)
            = \propmods(\phi \land \psi),
        \]
        which shows $W, c \models E_i(\phi \land \psi)$.

    \item[\labelcref{item:exp_on_all_variables}.]
        Let $\phi$ be a propositional formula, and let $p_1, \ldots, p_k$ be
        the variables appearing in $\phi$. Let $\widehat{\lprop} \subseteq
        \lprop$ be the propositional formulas over $p_1, \ldots p_k$ generated
        only using conjunction and negation. Then there is some $\psi \in
        \widehat{\lprop}$ with $\phi \equiv \psi$.

        Suppose $W, c \models E_i(p_1) \land \cdots \land E_i(p_k)$. By this
        assumption and the properties in (\labelcref{item:e_symmetric}),
        one can show by induction that $W, c \models E_i(\theta)$ for all
        $\theta \in \widehat{\lprop}$. In particular, $W, c \models E_i(\psi)$.
        Since $\phi \equiv \psi$, we get $W, c \models E_i(\phi)$.

    \item[\labelcref{item:e_and_s_implies_phi}.]
        Suppose $W, c \models E_i(\phi) \land S_i(\phi)$. Then $v_c \in
        \Pi_i[\phi] = \propmods(\phi)$, so $W, c \models \phi$. Hence
        $E_i(\phi) \land S_i(\phi) \rightarrow \phi$ is valid. Similarly,
        $S_i(\phi) \land \neg\phi \rightarrow \neg E_i(\phi)$ is valid.

    \item[\labelcref{item:sound_neg_pair}.]
        Suppose $W, c \models S_i(\phi) \land S_i(\neg\phi)$, and, for
        contradiction, $W, c \models E_i(\phi)$. On the one hand we have $W, c
        \models E_i(\phi) \land S_i(\phi)$, so
        (\labelcref{item:e_and_s_implies_phi}) gives $W, c \models \phi$. On
        the other hand, $W, c \models E_i(\phi)$ gives $W, c \models
        E_i(\neg\phi)$ by (\labelcref{item:e_symmetric}), so $W, c \models
        E_i(\neg\phi) \land S_i(\neg\phi)$; by
        (\labelcref{item:e_and_s_implies_phi}) again we have $W, c \models
        \neg\phi$. But then $W, c \models \phi \land \neg\phi$ --
        contradiction.

    \item[\labelcref{item:star_exp}.]
        Since the distinguished source $\ast$ has the unit partition $\Pi_\ast$
        in any world $W$, we have $\Pi_\ast[\phi] = \propmods(\phi)$, so $W, c
        \models E_\ast(\phi)$. Similarly, $W, c \models S_i(\phi)$ iff $v_c \in
        \Pi_\ast[\phi] = \propmods(\phi)$ iff $W, c \models \phi$.

\end{itemize}
\end{proof}

\subsection{Proof of \cref{prop:elementary_characterisation}}

\begin{proof}[Proof (sketch)]
    ``if": Suppose the stated condition holds for $S \subseteq \W$. Form a
    collection $G = \{\Gamma_c\}_{c \in \C}$ by setting $\Gamma_c = \{\Phi \in
    \lext \mid S \subseteq \mod_c(\Phi)\}$. Clearly $S \subseteq \mod(G)$. For the
    reverse inclusion, suppose $W \in \mod(G)$. For any set of valuations $U
    \subseteq \vals$, let $\phi_U$ be any propositional sentence with
    $\propmods(\phi_U) = U$. For each $c \in \C$, consider the sentence
    \[
        \Phi_c = \bigvee_{W' \in S}\left(
            \phi_{\{v^{W'}_c\}}
            \land
            \bigwedge_{i \in \S}
                \bigwedge_{U \subseteq \vals}
                    R_{W', i, U}
        \right)
    \]
    where
    \[
        R_{W', i, U} = \begin{cases}
            E_i(\phi_U),& W', c_0 \models E_i(\phi_U) \\
            \neg E_i(\phi_U),& \text{ otherwise }
        \end{cases}
    \]
    for some fixed case $c_0 \in \C$. It is straightforward to see that each
    $W' \in S$ satisfies its corresponding disjunct at case $c$, so $\Phi_c \in
    \Gamma_c$. Hence $W \in \mod(G)$ implies $W, c \models \Phi_c$ for each $c$.
    Consequently, for each $c$ there is some $W_c \in S$ such that
    \begin{inlinelist}
        \item \label{item:vals_equal} $v^W_c = v^{W_c}_c$; and
        \item \label{item:partition_equiv} for each $i \in \S$ and $U \subseteq
              V$, $W, c \models E_i(\phi_U)$ iff $W_c, c \models E_i(\phi_U)$.
    \end{inlinelist}
    From \ref{item:vals_equal}, $W$ is a valuation combination from $\{W_c\}_{c
    \in \C}$. From \ref{item:partition_equiv} it can be shown that in fact
    $\Pi^W_i = \Pi^{W_c}_i$ for each $c$ and $i$; that is, $W$ is
    partition-equivalent to each $W_c$. In particular, all the $W_c$ are
    partition-equivalent to each other.
    Repeatedly applying the closure condition assumed to hold for $S$, we see
    that $W \in S$ as required.

    ``only if": Suppose $S$ is elementary, i.e. $S = \mod(G)$ for some
    collection $G = \{\Gamma_c\}_{c \in \C}$, and let $W, W_1, W_2$ be as in
    the statement of the \lcnamecref{prop:elementary_characterisation}. Take $c
    \in \C$ and $\Phi \in \Gamma_c$. We will show $W, c \models \Phi$. By assumption,
    there is $n \in \{1, 2\}$ such that $v^W_c = v^{W_n}_c$. It can be shown by
    induction on $\lext$ formulas that, since $W$ and $W_n$ are
    partition-equivalent and have the same $c$ valuation, $W, c \models \Phi$ iff
    $W_n, c \models \Phi$. But $W_n \in S = \mod(G)$ implies $W_n, c \models
    \Phi$,
    so $W, c \models \Phi$ too.  Since $\Phi \in \Gamma_c$ was arbitrary, we have $W
    \in \mod(G) = S$ as required.

\end{proof}

\subsection{Proof of \cref{prop:prior_knowledge}}

\begin{proof}
    \leavevmode
    \begin{enumerate}
        \item The ``$\sqsubseteq$" inclusion is just \kbound{}.
              For the ``$\sqsupseteq$" inclusion, note that $G^\sigma_\snd
              \sqsubseteq K^\sigma$ by \soundness{}, and $K^\emptyset
              \sqsubseteq K^\sigma$ by \priorext{}. Hence
              \[
                  G^\sigma_\snd \sqcup K^\emptyset
                  \sqsubseteq
                  K^\sigma.
              \]
              By monotonicity of $\cn$,
              \[
                  \cn(G^\sigma_\snd \sqcup K^\emptyset)
                  \sqsubseteq
                  \cn(K^\sigma)
                  = K^\sigma
              \]
              where we use \closure{} in the final step.

          \item ``if": Suppose $K^\sigma = \cn(G^\sigma_\snd)$ for all
                $\sigma$. Taking $\sigma = \emptyset$ we obtain
                \[
                    K^\sigma = \cn(G^\emptyset_\snd) = \cn(\emptyset).
                \]
                ``only if": Suppose $K^\emptyset = \cn(\emptyset)$. Take any
                sequence $\sigma$. By \kbound{},
                \[
                    K^\sigma
                    \sqsubseteq \cn(G^\sigma_\snd \sqcup \cn(\emptyset))
                    = \cn(G^\sigma_\snd)
                \]
                On the other hand, \soundness{} and \closure{} give
                $\cn(G^\sigma_\snd) \sqsubseteq K^\sigma$. Hence $K^\sigma =
                \cn(G^\sigma_\snd)$.
    \end{enumerate}
\end{proof}

\subsection{Proof of \cref{prop:kconj}}

\begin{proof}
    Suppose an operator satisfies the basic postulates, and take sequences
    $\sigma$ and $\rho$. By \cref{prop:prior_knowledge},
    \[
        K^{\sigma \concat \rho}
        =
        \cn(G^{\sigma \concat \rho}_\snd \sqcup K^\emptyset)
    \]
    Note that $G^{\sigma \concat \rho}_\snd = G^\sigma_\snd \sqcup
    G^\rho_\snd$. Hence we may write
    \begin{align*}
        K^{\sigma \concat \rho}
        &= \cn(G^\sigma_\snd \sqcup G^\rho_\snd \sqcup K^\emptyset) \\
        &= \cn((G^\sigma_\snd \sqcup K^\emptyset) \sqcup (G^\rho_\snd \sqcup K^\emptyset))
    \end{align*}
    By \cref{prop:prior_knowledge} again, we have $K^\sigma = \cn(G^\sigma_\snd
    \sqcup K^\emptyset)$ and $K^\rho = \cn(G^\rho_\snd \sqcup K^\emptyset)$. It
    is easily verified that for any collections $G, D$, we have
    \[
        \cn(G \sqcup D) = \cn(\cn(G) \sqcup \cn(D)).
    \]
    Consequently,
    \begin{align*}
        K^{\sigma \concat \rho}
        &= \cn(\cn(G^\sigma_\snd \sqcup K^\emptyset) \sqcup \cn(G^\rho_\snd \sqcup K^\emptyset)) \\
        &= \cn(K^\sigma \sqcup K^\rho)
    \end{align*}
    as required for \kconj{}.
\end{proof}

\subsection{Proof of \cref{prop:basic_postulates_consequences}}

\begin{proof}
    \leavevmode
    \begin{enumerate}

        \item Suppose $\phi \in K^\sigma_c$ and $\neg\psi \in \cnprop(\phi)$.
              Write $\rho = \sigma \concat \tuple{i, c, \psi}$. By
              \soundness{}, $S_i(\psi) \in K^\rho_c$. By \kconj{}, $\phi \in
              K^\sigma_c \subseteq (K^\sigma \sqcup K^{\tuple{i, c, \psi}})_c
              \subseteq \cn_c(K^\sigma \sqcup K^{\tuple{i, c, \psi}}) =
              K^\rho_c$. Since $\neg\psi \in \cnprop(\phi)$ and $\phi \in
              K^\rho_c$, \closure{} gives $\neg\psi \in K^\rho_c$. Recalling
              from \cref{prop:validities} part
              (\labelcref{item:e_and_s_implies_phi}) that $S_i\psi \land \neg
              \psi \rightarrow \neg E_i(\psi)$, \closure{} gives $\neg
              E_i(\psi) \in K^\rho_c$, as desired.

          \item Suppose $\tuple{i, c, \phi} \in \sigma$ and $E_i(\phi) \in
                B^\sigma_c$. By \soundness{} and \containment{}, $S_i(\phi) \in
                B^\sigma_c$. From \cref{prop:validities} part
                (\labelcref{item:e_and_s_implies_phi}) again we have $E_i(\phi)
                \land S_i(\phi) \rightarrow \phi$. By \closure{}, $\phi \in
                B^\sigma_c$.

    \end{enumerate}
\end{proof}

\subsection{Proof of \cref{thm:model_based_characterisation}}

\begin{proof}
    For ease of notation in what follows, write $\mods_c(\Phi) = \{W \in \W
    \mid W, c \models \Phi\}$.

``if": Suppose an operator $\sigma \mapsto \tuple{B^\sigma, K^\sigma}$ is
model-based. For \closure{}, we need to show that $B^\sigma_c \supseteq
\cn_c(B^\sigma)$ and $K^\sigma_c \supseteq \cn_c(K^\sigma)$, for each $c$. Take
any $\Phi \in \cn_c(B^\sigma)$. Then $\mods(B^\sigma) \subseteq \mods_c(\Phi)$.
From the relevant definitions, one can easily check that $\Y_\sigma \subseteq
\mods(B^\sigma)$, so we have $\Y_\sigma \subseteq \mods_c(\Phi)$. That is, $W,
c \models \Phi$ for all $W \in \Y_\sigma$. By definition of model-based
operators, $\Phi \in B^\sigma_c$. The fact that $K^\sigma_c \supseteq
\cn_c(K^\sigma)$ follows by an identical argument upon noticing that $\X_\sigma
\subseteq \mods(K^\sigma)$.

\containment{} follows from $\X_\sigma \supseteq \Y_\sigma$: if $\Phi \in
K^\sigma_c$ then $W, c \models \Phi$ for all $W \in \X_\sigma$, and in
particular this holds for all $W \in \Y_\sigma$. Hence $\Phi \in B^\sigma_c$,
so $K^\sigma \sqsubseteq B^\sigma$.

``only if": Suppose an operator satisfies \closure{} and
\containment{}. For any $\sigma$, set
\[
    \X_\sigma = \mods(K^\sigma)
\]
\[
    \Y_\sigma = \mods(B^\sigma)
\]
We show the three properties required in \cref{def:model_based}.  $\X_\sigma
\supseteq \Y_\sigma$ follows from \containment{} and the definition of a
model of a collection. For the second property, note that $\Phi \in K^\sigma_c$
iff $\Phi \in \cn_c(K^\sigma)$ by \closure{}, i.e. iff $\mods(K^\sigma)
\subseteq \mods_c(\Phi)$.  By choice of $\X_\sigma$, this holds exactly when
$W, c \models \Phi$ for all $W \in \X_\sigma$, as required. The third property
is proved using an identical argument.
\end{proof}

\subsection{Proof of \cref{thm:conditioning_characterisation}}

First, a preliminary result.

\begin{lemma}
\label{lemma:model_based_elementary}
    For any model-based operator and sequence $\sigma$, $\X_\sigma =
    \mod(K^\sigma)$ iff $\X_\sigma$ is elementary, and $\Y_\sigma =
    \mod(B^\sigma)$ iff $\Y_\sigma$ is elementary.
\end{lemma}

\begin{proof}
We prove the result for $\X_\sigma$ and $K^\sigma$ only. The ``only if"
direction is clear from the definition of an elementary set. For the ``if"
direction, suppose $\X_\sigma$ is elementary, i.e. $\X_\sigma = \mod(G)$ for
some collection $G$. Since $\Phi \in K^\sigma_c$ iff $\X_\sigma \subseteq
\mod_c(\Phi)$, we have $K^\sigma_c = \cn_c(G)$, i.e. $K^\sigma = \cn(G)$.
Consequently $\mod(K^\sigma) = \mod(\cn(G)) = \mod(G) = \X_\sigma$.
\end{proof}

A result slightly more general than \cref{thm:conditioning_characterisation} is
available, from which \cref{thm:conditioning_characterisation} immediately
follows.

\begin{proposition}
    \label{prop:conditioning_pre_characterisation}
    Suppose an operator satisfies \closure{}, \containment{},
    \kconj{} and \equivpost{}. Then it is an elementary
    conditioning operator if and only if it satisfies \rearr{},
    \duprem{}, \condcons{},
    \incvac{} and \acyc{}.
\end{proposition}

\begin{proof}

Take some operator $\sigma \mapsto \tuple{B^\sigma, K^\sigma}$ satisfying
\closure{}, \containment{}, \kconj{} and
\equivpost{}.

``if": Suppose the operator in question additionally satisfies
\rearr{}, \duprem{}, \condcons{},
\incvac{} and \acyc{}. For any $\sigma$, set
\[ \X_\sigma = \mods(K^\sigma) \] \[ \Y_\sigma = \mods(B^\sigma) \]
Then -- by \closure{} and \containment{} as shown in the proof of
\cref{thm:model_based_characterisation} -- our operator is model based
corresponding to this choice of $\X_\sigma$ and $\Y_\sigma$. Clearly both are
elementary. We will construct a total preorder $\le$ over $\W$ such that
$\Y_\sigma = \min_{\le}{\X_\sigma}$; this will show the operator is an
elementary conditioning operator.

First, fix a function $c: \lprop / {\equiv} \to \lprop$ which chooses a fixed
representative of each equivalence class of logically equivalent propositional
formulas, i.e. any mapping such that $c([\phi]_{\equiv}) \equiv \phi$. To
simplify notation, write $\widehat{\phi}$ for $c([\phi]_{\equiv})$. Then $\phi
\equiv \widehat{\phi}$. Write $\widehat{\lprop} = \{\widehat{\phi} \mid \phi
\in \lprop\}$.  Note that $\widehat{\lprop}$ is finite (since we work with only
finitely many propositional variables) and every formula in $\lprop$ is
equivalent to exactly one formula in $\widehat{\lprop}$. For a sequence
$\sigma$, let $\widehat{\sigma}$ be the result of replacing each report
$\tuple{i, c, \phi}$ with $\tuple{i, c, \widehat{\phi}}$. Note that by
\rearr{} and \equivpost{}, $\X_{\widehat{\sigma}} =
\X_\sigma$ and $\Y_{\widehat{\sigma}} = \Y_\sigma$.

Now, for any world $W$, set
\newcommand{\reports}{\mathcal{R}}
\[
    \reports(W)
    = \{
        \tuple{i, c, \phi}
        \in \S \times \C \times \widehat{\lprop}
        \mid
        W \in \X_{\tuple{i, c, \phi}}
    \}
\]
Note that $\reports(W)$ is finite. For any pair of worlds $W_1$, $W_2$, let
$\rho(W_1, W_2)$ be some enumeration of $\reports(W_1) \cap \reports(W_2)$. We
establish some useful properties of $\rho(W_1, W_2)$.

\begin{addmargin}[1em]{1em}
    \begin{claim}
        \label{claim:w_j_in_rho}
        If $\rho(W_1, W_2) \ne \emptyset$, $W_1, W_2 \in \X_{\rho(W_1, W_2)}$.
    \end{claim}
    \begin{proof}
        By \kconj{}, for any sequences $\sigma$, $\rho$ we have
        $K^{\sigma \concat \rho} = \cn(K^\sigma \sqcup K^\rho)$.  Taking the
        models of both sides, we have $\X_{\sigma \concat \rho} = \X_\sigma
        \cap \X_\rho$. It follows that for $\rho(W_1, W_2) \ne \emptyset$,
        \[
            \X_{\rho(W_1, W_2)}
            = \bigcap_{\tuple{i, c, \phi} \in \rho(W_1, W_2)}{\X_{\tuple{i, c,
            \phi}}}
        \]
        If $\tuple{i, c, \phi} \in \rho(W_1, W_2)$ then $W_1, W_2 \in
        \X_{\tuple{i, c, \phi}}$ by definition. Hence $W_1, W_2 \in
        \X_{\rho(W_1, W_2)}$.
    \end{proof}
\end{addmargin}

\begin{addmargin}[1em]{1em}
    \begin{claim}
        \label{claim:exists_delta}
        If a sequence $\sigma$ contains no equivalent reports (i.e. no distinct
        tuples $\tuple{i, c, \phi}$, $\tuple{i, c, \psi}$ with $\phi \equiv
        \psi$) and $W_1, W_2 \in \X_\sigma$, there is a sequence $\delta$ such
        that $W_1, W_2 \in \X_\delta$ and $\rho(W_1, W_2)$ is a permutation of
        $\widehat{\sigma} \concat \delta$.
    \end{claim}
    \begin{proof}
        If $\sigma = \emptyset$ then we can simply take $\delta = \rho(W_1,
        W_2)$. So suppose $\sigma \ne \emptyset$. By the same argument as in
        the proof of \cref{claim:w_j_in_rho}, we have
        \[
            \X_\sigma = \bigcap_{\tuple{i, c, \phi} \in \sigma}{\X_{\tuple{i, c,
            \phi}}}
        \]
        Take any $\tuple{i, c, \phi} \in \widehat{\sigma}$. Then $\phi \in
        \widehat{\lprop}$, and there is $\psi \equiv \phi$ such that $\tuple{i,
        c, \psi} \in \sigma$. By \equivpost{}, we have
        \[
            W_1, W_2
            \in \X_\sigma
            \subseteq \X_{\tuple{i, c, \psi}}
            = \X_{\tuple{i, c, \phi}}
        \]
        i.e. $\tuple{i, c, \phi} \in \reports(W_1) \cap \reports(W_2)$. Hence
        $\tuple{i, c, \phi}$ appears in $\rho(W_1, W_2)$. By the assumption
        that $\sigma$ contains no equivalent reports, $\widehat{\sigma}$
        contains no duplicates. It follows that $\rho(W_1, W_2)$ can be
        permuted so that $\widehat{\sigma}$ appears as a prefix. Taking
        $\delta$ to be the sequence that remains after $\widehat{\sigma}$ in
        this permutation, we clearly have that $\rho(W_1, W_2)$ is a
        permutation of $\widehat{\sigma} \concat \delta$. Since $\sigma \ne
        \emptyset$ implies $\widehat{\sigma} \ne \emptyset$ and thus $\rho(W_1,
        W_2) \ne \emptyset$, by \rearr{}, \kconj{} and
        \cref{claim:w_j_in_rho} we get
        \[
            W_1, W_2 \in \X_{\rho(W_1, W_2)}
            = \X_{\widehat{\sigma} \concat \delta}
            = \X_{\widehat{\sigma}} \cap \X_\delta
            \subseteq \X_\delta
        \]
        and we are done.
    \end{proof}
\end{addmargin}

Now define a relation $R$ on $\W$ by
\[
    W R W' \iff W = W' \text{ or } W \in \Y_{\rho(W, W')}
\]
We have that any world in $\Y_\sigma$ $R$-precedes all worlds $\X_\sigma$.

\begin{addmargin}[1em]{1em}
    \begin{claim}
        \label{claim:y_subset_minimal}
        If $W \in \Y_\sigma$, then for all $W' \in \X_\sigma$ we have $W R W'$
    \end{claim}
    \begin{proof}
        By \rearr{}, \equivpost{} and
        \duprem{}, we may assume without loss of generality that
        $\sigma$ contains no distinct equivalent reports.

        Let $W \in \Y_\sigma$ and $W' \in \X_\sigma$. Then $W \in \X_\sigma$
        too. By \cref{claim:exists_delta} and \rearr{}, there is
        some sequence $\delta$ such that $\Y_{\rho(W, W')} =
        \Y_{\widehat{\sigma} \concat \delta}$ and $W, W' \in \X_\delta$.
        Consequently $W \in \Y_\sigma \cap \X_\delta = \Y_{\widehat{\sigma}}
        \cap \X_\delta$. Thus $B^{\widehat{\sigma}} \sqcup K^\delta$ is
        consistent. From \incvac{} we get
        \[
            \Y_{\widehat{\sigma} \concat \delta} = \Y_{\widehat{\sigma}} \cap
            \X_\delta
        \]
        Thus
        \[
            W \in
            \Y_{\widehat{\sigma}} \cap \X_\delta
            = \Y_{\widehat{\sigma} \concat \delta}
            = \Y_{\rho(W, W')}
        \]
        so $W R W'$ as required.
    \end{proof}
\end{addmargin}

Now let $\le_0$ be the transitive closure of $R$. Then $\le_0$ is a (partial)
preorder. By \cref{claim:y_subset_minimal}, every world in $\Y_\sigma$ is
$\le_0$-minimal in $\X_\sigma$. In fact, the converse is also true.

\begin{addmargin}[1em]{1em}
    \begin{claim}
        \label{claim:minimal_subset_y}
        If $W \in \X_\sigma$ and there is no $W' \in \X_\sigma$ with $W' <_0
        W$, then $W \in \Y_\sigma$.
    \end{claim}
    \begin{proof}
        As before, assume without loss of generality that
        $\sigma$ contains no distinct equivalent reports.

        Take $W$ as in the statement of the claim. Then $\X_\sigma \ne
        \emptyset$, so $\Y_\sigma \ne \emptyset$ by
        \condcons{}. Let $W' \in \Y_\sigma$.  By
        \cref{claim:y_subset_minimal}, $W' R W$ and thus $W' \le_0 W$. But by
        assumption, $W' \not<_0 W$. So we must have $W \le_0 W'$. By definition
        of $\le_0$ as the transitive closure of $R$, there are $W_0, \ldots
        W_n$ such that $W_0 = W$, $W_n = W'$ and
        \[
            W_j R W_{j+1} \qquad (0 \le j < n)
        \]
        Without loss of generality, $n > 0$ and each of the $W_j$ are distinct.
        From the definition of $R$, we therefore have that
        \[
            W_j \in \rho(W_j, W_{j+1}) \qquad (0 \le j < n)
        \]
        Now set
        \begin{align*}
            \rho_j &= \rho(W_j, W_{j+1}) \qquad (0 \le j < n) \\
            \rho_n &= \rho(W_0, W_n)
        \end{align*}
        Since $W' R W$, i.e. $W_n R W_0$, we in fact have $W_j \in \Y_{\rho_j}$
        for all $j$ (including $j = n$). For $j < n$, we also have $W_{j+1} \in
        \X_{\rho_j}$.\footnotemark{}
        \footnotetext{
            If $\rho_j \ne \emptyset$ this follows from
            \cref{claim:w_j_in_rho}. Otherwise, $W_{j+1} \in \Y_{\rho_{j+1}}
            \subseteq \X_{\rho_{j+1}} = \X_{\rho_{j+1} \concat \emptyset} =
            \X_{\rho_{j+1}} \cap \X_\emptyset \subseteq \X_\emptyset =
            \X_{\rho_j}$ by \kconj{}.
        }
        Consequently, for $j < n$ we have
        \[
            W_{j+1} \in \X_{\rho_j} \cap \Y_{\rho_{j+1}}
        \]
        i.e. $K^{\rho_j} \sqcup B^{\rho_{j+1}}$ is consistent. Moreover, $W_0
        \in \X_{\rho_n} \cap \Y_{\rho_0}$, so $K^{\rho_n} \sqcup B^{\rho_0}$ is
        consistent. We can now apply \acyc{}: we get that $K^{\rho_0}
        \sqcup B^{\rho_n}$ is also consistent. On the one hand,
        \incvac{} and consistency of $K^{\rho_n} \sqcup
        B^{\rho_0}$ gives
        \[
            B^{\rho_0 \concat \rho_n}
            = \cn(B^{\rho_0} \sqcup K^{\rho_n})
        \]
        On the other, consistency of $B^{\rho_n} \sqcup K^{\rho_0}$ and
        \rearr{} gives
        \[
            B^{\rho_0 \concat \rho_n}
            = B^{\rho_n \concat \rho_0}
            = \cn(B^{\rho_n} \sqcup K^{\rho_0})
        \]
        Combining these and taking models, we find
        \[
            \Y_{\rho_0} \cap \X_{\rho_n}
            =
            \Y_{\rho_n} \cap \X_{\rho_0}
        \]
        In particular, since $W_0$ lies in the set on the left-hand side, we
        have $W_0 \in \Y_{\rho_n}$.

        Now, since $W_0, W_n \in \X_\sigma$ and $\rho_n = \rho(W_0, W_n)$,
        \cref{claim:exists_delta} gives that there is $\delta$ with $W_0, W_n
        \in \X_{\delta}$ such that $\rho_n$ is a permutation of
        $\widehat{\sigma} \concat \delta$. Recalling that $W_n = W' \in
        \Y_\sigma = \Y_{\widehat{\sigma}}$ by assumption, we have $W_n \in
        \Y_{\widehat{\sigma}} \cap \X_\delta$, i.e. $B^{\widehat{\sigma}}
        \sqcup K^\delta$ is consistent. Applying \incvac{} once
        more, we get
        \[
            B^{\rho_n}
            = B^{\widehat{\sigma} \concat \delta}
            = \cn(B^{\widehat{\sigma}} \sqcup K^\delta)
            = \cn(B^{\sigma} \sqcup K^\delta)
        \]
        Taking models of both sides,
        \[
            \Y_{\rho_n} = \Y_{\sigma} \cap \X_\delta \subseteq \Y_\sigma
        \]
        But we already saw that $W_0 \in \Y_{\rho_n}$. Hence $W_0 \in
        \Y_\sigma$. Since $W_0 = W$, we are done.
    \end{proof}
\end{addmargin}

To complete the proof we extend $\le_0$ to a \emph{total} preorder and show
that this does not affect the minimal elements of each $\X_\sigma$. Indeed, let
$\le$ be any total preorder extending $\le_0$ and preserving strict
inequalities, i.e. $\le$ such that
\begin{inlinelist}
    \item $W \le_0 W'$ implies $W \le W'$; and
    \item\label{item:strict_ineq_preserved} $W <_0 W'$ implies $W < W'$.
\end{inlinelist}\footnote{
    Such $\le$ always exists. Indeed, note that $\le_0$ induces a partial order
    on the equivalence classes of $\W$ with respect to the symmetric part of
    $\le_0$ given by $W \simeq_0 W'$ iff $W \le_0 W'$ and $W' \le_0 W$. This
    partial order can be extended to a linear order $\le^*$ on the equivalence
    classes. Taking $W \le W'$ iff $[W]  \le^* [W']$, we obtain a total
    preorder on $\W$ with the desired properties.
}

\begin{addmargin}[1em]{1em}
    \begin{claim}
        For any sequence $\sigma$, $\Y_\sigma = \min_{\le}{\X_\sigma}$
    \end{claim}
    \begin{proof}
        Take any $\sigma$. For the left-to-right inclusion, take $W \in
        \Y_\sigma$. Then $W \in \X_\sigma$. Let $W' \in \X_\sigma$. By
        \cref{claim:y_subset_minimal}, $W R W'$, so $W \le_0 W'$ and $W \le
        W'$. Hence $W$ is $\le$-minimal in $\X_\sigma$.

        For the right-to-left inclusion, take $W \in \min_{\le}{\X_\sigma}$.
        Then for any $W' \in \X_\sigma$ we have $W \le W'$. In particular, $W'
        \not< W$. By property \labelcref{item:strict_ineq_preserved} of $\le$,
        we have $W' \not<_0 W$. Since $W'$ was an arbitrary member of
        $\X_\sigma$ and $W \in \X_\sigma$, the conditions of
        \cref{claim:minimal_subset_y} are satisfied, and we get $W \in
        \Y_\sigma$.
    \end{proof}
\end{addmargin}

This shows that our operator is an elementary conditioning operator as
required.

``only if": Now  suppose the operator is an elementary conditioning operator.
i.e. there is a total preorder $\le$ on $\W$ and a mapping $\sigma \mapsto
\tuple{\X_\sigma, \Y_\sigma}$ such that for each $\sigma$, $\Y_\sigma =
\min_{\le}{X_\sigma}$, $\X_\sigma$ and $\Y_\sigma$ are elementary, and
$K^\sigma$, $B^\sigma$ are determined by $\X_\sigma$, $\Y_\sigma$ respectively
according to \cref{def:model_based}. By elementariness and
\cref{lemma:model_based_elementary}, $\X_\sigma = \mod(K^\sigma)$ and $\Y_\sigma
= \mod(B^\sigma)$.

The following claim will be useful at various points.

\begin{addmargin}[1em]{1em}
    \begin{claim}
        \label{claim:x_equal_implies_output_equal}
        Suppose $\sigma$ and $\rho$ are such that $\X_\sigma = \X_\rho$. Then
        $K^\sigma = K^\rho$ and $B^\sigma = B^\rho$.
    \end{claim}
    \begin{proof}
        Since the total preorder $\le$ is fixed, we have
        \[
            Y_\sigma
            = \min\nolimits_{\le}{\X_\sigma}
            = \min\nolimits_{\le}{\X_\rho}
            = \Y_\rho
        \]
        Now, $\X_\sigma = \X_\rho$ means $\mod(K^\sigma) = \mod(K^\rho)$, so
        $\cn(K^\sigma) = \cn(K^\rho)$. By \closure{}, $K^\sigma = K^\rho$.
        Similarly, $\Y_\sigma = \Y_\rho$ gives $B^\sigma = B^\rho$.
    \end{proof}
\end{addmargin}

We take the postulates to be shown in turn.

\begin{itemize}
    \item \rearr{}: Suppose $\sigma$ is a permutation of $\rho$.
          Without loss of generality, $\sigma, \rho \ne \emptyset$. Repeated
          application of \kconj{} gives
          \[
            \X_{\sigma} = \bigcap_{\tuple{i, c, \phi} \in \sigma}{\X_{\tuple{i,
            c, \phi}}}
          \]
          Since $\sigma$ and $\rho$ contain exactly the same reports -- just in
          a different order -- commutativity and associativity of intersection
          of sets gives $\X_\sigma = \X_\rho$. \rearr{} follows
          from \cref{claim:x_equal_implies_output_equal}.

    \item \duprem{}: Let $\sigma$, $\rho_1$ and $\rho_2$ be as
          in the statement of \duprem{}. Then by
          \kconj{},
          \begin{align*}
            \X_{\rho_2}
            &= \X_{\rho_1 \concat \tuple{i, c, \phi}} \\
            &= \X_{\rho_1} \cap \X_{\tuple{i, c, \phi}} \\
            &= \X_{\sigma \concat \tuple{i, c, \phi}} \cap \X_{\tuple{i, c,
            \phi}} \\
            &= \X_{\sigma} \cap \X_{\tuple{i, c, \phi}} \cap \X_{\tuple{i, c,
            \phi}} \\
            &= \X_{\sigma} \cap \X_{\tuple{i, c, \phi}} \\
            &= \X_{\rho_1}
          \end{align*}
          and we may conclude by \cref{claim:x_equal_implies_output_equal}.
    \item \condcons{}: Suppose $K^\sigma$ is consistent,
          i.e. $\X_\sigma \ne \emptyset$. Since $\W$ is finite, $\X_\sigma$ is
          finite and thus some $\le$-minimal world must exist in $\X_\sigma$.
          Hence $\Y_\sigma \ne \emptyset$, so $B^\sigma$ is consistent.
    \item \incvac{}: Take any sequences $\sigma$, $\rho$. First
          we show $B^{\sigma \concat \rho} \sqsubseteq \cn(B^\sigma \sqcup
          K^\rho)$, or equivalently, $\Y_{\sigma \concat \rho} \supseteq
          \Y_\sigma \cap \X_\rho$. Suppose $W \in \Y_\sigma \cap \X_\rho$.
          Since $\Y_\sigma \subseteq \X_\sigma$, we have $W \in \X_\sigma \cap
          \X_\rho = \X_{\sigma \concat \rho}$ by \kconj{}. We need
          to show $W$ is minimal. Take any $W' \in \X_{\sigma \concat \rho}$.
          Then $W' \in \X_{\sigma}$, so $W \in \Y_\sigma =
          \min_{\le}{\X_\sigma}$ gives $W \le W'$. Hence $W \in
          \min_{\le}{X_{\sigma \concat \rho}} = \Y_{\sigma \concat \rho}$.

          Now suppose $B^\sigma \sqcup K^\rho$ is consistent, i.e. $\Y_\sigma
          \cap \X_\rho \ne \emptyset$. Take some $\widehat{W} \in \Y_\sigma
          \cap \X_\rho$. We need to show $B^{\sigma \concat \rho} \sqsupseteq
          \cn(B^\sigma \sqcup K^\rho)$, i.e. $\Y_{\sigma \concat \rho}
          \subseteq \Y_\sigma \cap \X_\rho$. To that end, let $W \in \Y_{\sigma
          \concat \rho}$. Then $W \in \X_{\sigma \concat \rho} = \X_\sigma \cap
          \X_\rho \subseteq \X_\rho$, so we only need to show $W \in
          \Y_\sigma$. Take any $W' \in \X_\sigma$. Then $\widehat{W} \in
          \Y_\sigma$ gives $\widehat{W} \le W'$. But $\widehat{W} \in \X_\sigma
          \cap \X_\rho = \X_{\sigma \concat \rho}$ and $W \in \Y_{\sigma
          \concat \rho}$ gives $W \le \widehat{W}$. By transitivity of $\le$,
          we have $W \le W'$.  Hence $W \in \min_{\le}{\X_\sigma} = \Y_\sigma$.

    \item \acyc{}: Let $\sigma_0, \ldots, \sigma_n$ be as in the statement
          of \acyc{}. Without loss of generality, $n > 0$. Then there are
          $W_0, \ldots, W_n$ such that
          \begin{align*}
            W_j &\in \X_{\sigma_j} \cap \Y_{\sigma_{j+1}} \qquad (0 \le j < n) \\
            W_n &\in \X_{\sigma_n} \cap \Y_{\sigma_0}
          \end{align*}
          Note that $W_j \in \X_{\sigma_j}$ for all $j$. For $j < n$, we also
          have $W_j \in \Y_{\sigma_{j+1}} = \min_{\le}{X_{\sigma_{j+1}}}$. It
          follows that $W_j \le W_{j+1}$ for such $j$, so
          \[
            W_0 \le \cdots \le W_n
          \]
          But we also have $W_n \in \Y_{\sigma_0} = \min_{\le}{\X_{\sigma_0}}$
          and $W_0 \in \X_{\sigma_0}$, so $W_n \le W_0$. By transitivity of
          $\le$, the chain flattens: we have
          \[
            W_0 \simeq \cdots \simeq W_n
          \]
          Now note that since $W_{n-1} \in \Y_{\sigma_n}$, $W_{n - 1}$ is
          minimal in $\X_{\sigma_n}$. But $W_n \in \X_{\sigma_n}$ and $W_{n -
          1} \simeq W_n$ by the above, so in fact $W_n \in \Y_{\sigma_n}$ too.
          Hence
          \begin{align*}
            W_n
            &\in \Y_{\sigma_0} \cap \Y_{\sigma_n} \\
            &\subseteq \X_{\sigma_0} \cap \Y_{\sigma_n} \\
            &= \mod(K^{\sigma_0} \sqcup B^{\sigma_n})
          \end{align*}
          i.e. $K^{\sigma_0} \sqcup B^{\sigma_n}$ is consistent, as required
          for \acyc{}.
\end{itemize}

\end{proof}

\subsection{Proof of \cref{prop:partition_equiv_tpo_implies_elementariness}}

\begin{proof}
    We use the characterisation of elementary sets from
    \cref{prop:elementary_characterisation}. Take $S \subseteq \W$ elementary.
    Suppose $W \in \W$, $W_1, W_2 \in \min_{\le}{S}$ are such that $W$ is
    partition-equivalent to both $W_1, W_2$ and $W$ is a valuation combination
    from $\{W_1, W_2\}$. By hypothesis we have $W \simeq W_1 \simeq W_2$.

    Now since $\min_{\le}{S} \subseteq S$, we have $W_1, W_2 \in S$. Since $S$
    is elementary, $W \in S$. But now $W \simeq W_1$ and $W_1 \in
    \min_{\le}{S}$ gives $W \in \min_{\le}{S}$. This shows the required closure
    property for $\min_{\le}{S}$, and we are done.
\end{proof}

\subsection{Proof of \cref{prop:kconj_conditioning_implies_score_based}}

\begin{proof}
    Suppose an operator with \kconj{} is an elementary conditioning operator
    corresponding to some mapping $\sigma \mapsto \tuple{\X_\sigma, \Y_\sigma}$
    and total preorder $\le$, i.e. where $\X_\sigma$ and $\Y_\sigma$ are
    elementary and $\Y_\sigma = \min_{\le}{\X_\sigma}$ for any $\sigma$. Write
    \[
        k(W) = |\{W' \in \W \mid W' \le W\}|
    \]
    Then we have $W \le W'$ iff $k(W) \le k(W')$. Set
    \[
        r_0(W) = \begin{cases}
            \infty,& W \notin \X_\emptyset \\
            k(W),& W \in \X_\emptyset
        \end{cases}
    \]
    and
    \[
        d(W, \tuple{i, c, \phi}) = \begin{cases}
            \infty,& W \notin \X_{\tuple{i, c, \phi}} \\
            0,& W \in \X_{\tuple{i, c, \phi}} \\
        \end{cases}
    \]
    For the empty sequence, we clearly have $r_\emptyset(W) < \infty$ iff $W
    \in \X_\emptyset$. By construction of $r_0$ we also have $\argmin_{W \in
    \X_\emptyset}{r_\emptyset(W)} = \min_{\le}{X_\emptyset} = \Y_\emptyset$.

    Now take any sequence $\sigma \ne \emptyset$. By elementariness and
    \cref{lemma:model_based_elementary}, we have $\mod(K^\rho) = \X_\rho$ for
    any $\rho$. Repeated applications of \kconj{} therefore gives
    \[
        \X_\sigma
        = \bigcap_{\tuple{i, c, \phi} \in \sigma}{\X_{\tuple{i, c, \phi}}}
    \]
    We also have $\X_\sigma = \X_{\sigma \concat \rho} = \X_\sigma \cap \X_\emptyset
    \subseteq \X_\emptyset$. Since $r_0(W) < \infty$ iff $W \in \X_\emptyset$
    and $d(W, \tuple{i, c, \phi}) < \infty$ iff $W \in \X_{\tuple{i, c,
    \phi}}$, it follows that $r_\sigma(W) < \infty$ iff $W \in \X_\sigma$.

    Finally, note that if $W \in \X_\sigma$ then $r_\sigma(W) = r_0(W)$. Hence
    $\argmin_{W \in \X_\sigma}{r_\sigma(W)} = \argmin_{W \in \X_\sigma}{r_0(W)}
    = \min_{\le}{\X_\sigma} = \Y_\sigma$, and the operator is score-based.
\end{proof}

\subsection{Proof of \cref{prop:examples_satisfy_agm}}

We require some preliminary results. For a case $c \in \C$ and valuation $v \in
\vals$, write $\W_{c\ :\ v} = \{W \in \W \mid v^W_c = v\}$ for the set of
worlds whose $c$ valuation is $v$.

\begin{lemma}
    \label{lemma:model_based_models_of_proppart}
    For any model-based operator, sequence $\sigma$, case $c$, and valuation
    $v\ in \vals$,
    \[
        v \in \propmods\proppart{B^\sigma_c}
        \iff
        \Y_\sigma \cap \W_{c\ :\  v} \ne \emptyset
    \]
\end{lemma}

\begin{proof}
    ``$\implies$": We show the contrapositive. Suppose $\Y_\sigma \cap \W_{c\
    :\ v} = \emptyset$. Let $\psi$ be any propositional formula such that
    $\propmods(\psi) = \vals \setminus \{v\}$. Now for any $W \in \Y_\sigma$,
    we have $W \notin \W_{c\ :\  v}$, i.e. $v^W_c \ne v$. Hence $v^W_c \in
    \propmods(\psi)$, so $W, c \models \psi$. By definition of the belief set
    of a model-based operator, we have $\psi \in B^\sigma_c$. But $\psi$ is a
    propositional formula, so $\psi \in \proppart{B^\sigma_c}$. Since $v \notin
    \propmods(\psi)$, we have $v \notin \propmods\proppart{B^\sigma_c}$.

    ``$\impliedby$": Suppose there is some $W \in \Y_\sigma \cap \W_{c\ :\
    v}$. Let $\phi \in \proppart{B^\sigma_c}$. Then, in particular, $\phi \in
    B^\sigma_c$, so $W, c \models \phi$ by $W \in \Y_\sigma$ and the definition
    of the model-based belief set. That is, $v = v^W_c \in \propmods(\phi)$.
    Since $\phi \in \proppart{B^\sigma_c}$ was arbitrary, we have $v \in
    \propmods\proppart{B^\sigma_c}$.

\end{proof}

We have a sufficient condition for \agm{} for score-based operators.

\begin{lemma}
    \label{lemma:score_based_agm_sufficient_conditions}
    Suppose a score-based operator is such that for each $c \in \C$ and $\phi
    \in \lprop$ there is a constant $M \in \mathbb{N}$ with
    \[
        d(W, \tuple{\ast, c,  \phi})
        = \begin{cases}
            M,& W, c \models \phi  \\
            \infty,& W, c \models \neg \phi
        \end{cases}
    \]
    for all $W$. Then \agm{} holds.
\end{lemma}

\begin{proof}
    Take a score-based operator with the stated property. Let $\sigma$ be a
    sequence and take $c \in \C$. Without loss of generality, there is some
    $\phi \in \lprop$ such that $\neg\phi \notin K^\sigma_c$ (otherwise \agm{}
    trivially holds). Since any score-based operator is model-based and
    therefore satisfies \closure{}, we have that $K^\sigma$ is inconsistent iff
    $K^\sigma_c = \lext$. But since $K^\sigma_c$ does not contain $\neg\phi$,
    it must be the case that $K^\sigma$ is consistent.

    Now, set
    \[
        k(v)
        = \min\{
            r_\sigma(W)
            \mid
            W \in \X_\sigma \cap \W_{c\ :\ v}
        \}
    \]
    where $\min\emptyset =  \infty$. Note that $k(v) = \infty$ if and only if
    $\X_\sigma \cap \W_{c\ :\  v} = \emptyset$. Then $k$ defines a total
    preorder $\preceq$ on valuations, where $v \preceq v'$ iff $k(v) \le
    k(v')$. Define a propositional revision operator $\star$ for
    $\proppart{B^\sigma_c}$ by
    \[
        \proppart{B^\sigma_c} \star \phi = \{
            \psi \in \lprop
            \mid
            \min\nolimits_{\preceq}{\propmods(\phi)} \subseteq \propmods(\psi)
        \}
    \]
    To show that $\star$ satisfies the AGM postulates (for
    $\proppart{B^\sigma_c}$) it is sufficient to show that the models of
    $\proppart{B^\sigma_c}$ are exactly the $\preceq$-minimal valuations.

    \begin{addmargin}[1em]{1em}
        \begin{claim}
            $\propmods\proppart{B^\sigma_c} = \min_{\preceq}{\vals}$.
        \end{claim}
        \begin{proof}

            ``$\subseteq$": let $v \in \propmods\proppart{B^\sigma_c}$. By
            \cref{lemma:model_based_models_of_proppart}, there is some $W \in
            \Y_\sigma \cap \W_{c\ :\  v}$. Since $W \in \X_\sigma$ too, by
            definition of $k$ we have $k(v) \le r_\sigma(W) < \infty$. Now let
            $v' \in \vals$. Without loss of generality assume $k(v') < \infty$.
            Then there is some $W' \in \X_\sigma \cap \W_{c\ :\  v'}$ such that
            $k(v') = r_\sigma(W')$. But $W' \in \X_\sigma$ and $W \in
            \Y_\sigma$ gives $r_\sigma(W) \le r_\sigma(W')$, so
            \[
                k(v) \le r_\sigma(W) \le r_\sigma(W') = k(v')
            \]
            i.e. $v \preceq v'$. Hence $v$ is $\preceq$-minimal.

            ``$\supseteq$": let $v \in \min_{\preceq}\vals$. Since $K^\sigma$
            is consistent, there is some $\hat{W} \in \X_\sigma$.  Writing
            $\hat{v} = v^{\hat{W}}_c$, we have $\hat{W} \in \X_\sigma \cap
            \W_{c\ :\ \hat{v}}$, so $v \preceq \hat{v}$ implies
            \[
                k(v)
                \le k(\hat{v})
                \le r_\sigma(\hat{W})
                < \infty
            \]
            Hence there must be some $W \in \X_\sigma \cap \W_{c\ :\ v}$ such
            that $k(v) = r_\sigma(W)$. We claim that, in fact, $W \in
            \Y_\sigma$. Indeed, for any $W' \in \X_\sigma$ we have $v \preceq
            v^{W'}_c$, so
            \[
                r_\sigma(W)
                = k(v)
                \le k(v^{W'}_c)
                \le r_\sigma(W')
            \]
            That is, $W \in \Y_\sigma \cap \W_{c\ :\  v}$. By
            \cref{lemma:model_based_models_of_proppart}, $v \in
            \propmods\proppart{B^\sigma_c}$.

        \end{proof}
    \end{addmargin}

    So, $\star$ is indeed an AGM operator for $\proppart{B^\sigma_c}$. Now take
    $\phi \in \lprop$ such that $\neg\phi \notin K^\sigma_c$. Write $\rho =
    \sigma \concat \tuple{\ast, c, \phi}$. We claim the following.

    \begin{addmargin}[1em]{1em}
        \begin{claim}
            \label{claim:propmods_rho_min_models_of_phi}
            $\propmods\proppart{B^\rho_c} =
            \min\nolimits_{\preceq}{\propmods(\phi)}$.
        \end{claim}
        \begin{proof}
            ``$\subseteq$": let $v \in \propmods\proppart{B^\rho_c}$.  By
            \cref{lemma:model_based_models_of_proppart} again, there is some $W
            \in \Y_\rho \cap \W_{c\ :\  v}$. Since $\tuple{\ast, c, \phi} \in
            \rho$ and $d(W, \tuple{\ast, c, \phi}) \le r_\rho(W) < \infty$, we
            must have $W, c \models \phi$ by the assumed property of the score
            function $d$. Hence $v = v^W_c \in \propmods(\phi)$.

            Now since $\Y_\rho \subseteq \X_\rho$, we have $W \in \Y_\rho
            \subseteq \X_\rho \subseteq \X_\sigma$, so $W \in \X_\sigma \cap
            \W_{c\ :\  v}$. By definition of $k$, we have $k(v) \le
            r_\sigma(W)$. Take any $v' \in \propmods(\phi)$. Without loss of
            generality, assume $k(v') < \infty$, so that there is some $W' \in
            \X_\sigma \cap \W_{c\ :\  v'}$ with $k(v') = r_\sigma(W')$. Since
            $v^{W'}_c = v' \in \propmods(\phi)$, we have $W', c \models \phi$.
            Consequently, by the property of $d$ again, $d(W', \tuple{\ast, c,
            \phi}) = M$. Since $W' \in \X_\sigma$ gives $r_\sigma(W') <
            \infty$, it follows that
            \[
                r_\rho(W') = r_\sigma(W') + M < \infty
            \]
            so $W' \in \X_\rho$.

            Recall that $W, c \models \phi$ too, so $d(W, \tuple{\ast, c,
            \phi}) = M$ also. From $W \in \Y_\rho$ and $W' \in \X_\rho$ we get
            \begin{align*}
                r_\sigma(W)
                &= r_\rho(W) - M \\
                &\le r_\rho(W') - M \\
                &= r_\rho(W') - d(W', \tuple{\ast, c, \phi}) \\
                &= r_\sigma(W')
            \end{align*}
            This yields
            \[
                k(v) \le r_\sigma(W) \le r_\sigma(W') = k(v')
            \]
            and $v \preceq v'$ as required.

            ``$\supseteq$": let $v \in \min_{\preceq}{\propmods(\phi)}$. Since
            $\neg\phi \notin K^\sigma_c$, there is some $\hat{W} \in \X_\sigma$
            such that $\hat{W}, c \models \phi$. Writing $\hat{v} =
            v^{\hat{W}}_c$, we have $\hat{v} \in \propmods(\phi)$. Hence $v
            \preceq \hat{v}$. This implies
            \[
                k(v) \le k(\hat{v}) \le r_\sigma(\hat{W}) < \infty
            \]
            so there must be some $W \in \X_\sigma \cap \W_{c\ :\  v}$ with
            $k(v) = r_\sigma(W)$. Since $v^W_c = v \in \propmods(\phi)$, we
            have $W, c \models \phi$. By the assumed property of $d$, we get
            $d(W, \tuple{\ast, c, \phi}) = M$. Hence
            \[
                r_\rho(W)
                = r_\sigma(W) + d(W, \tuple{\ast, c, \phi})
                = r_\sigma(W) + M
                < \infty
            \]
            so $W \in \X_\rho$ too. We will show that $W \in \Y_\rho$. Let $W'
            \in \X_\rho$. Then we must have $d(W', \tuple{\ast, c, \phi}) = M$
            and $W', c \models \phi$. That is, $v^{W'}_c \in \propmods(\phi)$.
            By minimality of $v$, we have $v \preceq v^{W'}_c$. Noting that $W'
            \in \X_\rho \subseteq \X_\sigma$, we get
            \[
                r_\sigma(W)
                = k(v)
                \le k(v^{W'}_c)
                \le r_\sigma(W')
            \]
            Consequently,
            \[
                r_\rho(W)
                = r_\sigma(W) + M
                \le r_\sigma(W') + M
                = r_\rho(W')
            \]
            This shows $W \in \Y_\rho$, i.e. $\Y_\rho \cap \W_{c\ :\  v} \ne
            \emptyset$. By \cref{lemma:model_based_models_of_proppart}, we are
            done.

        \end{proof}
    \end{addmargin}

    Noting that $\propmods\proppart{B^\sigma_c} \star \phi =
    \min_{\preceq}{\propmods(\phi)}$, it follows from
    \cref{claim:propmods_rho_min_models_of_phi} that
    $\cnprop(\proppart{B^\rho_c}) = \cnprop(\proppart{B^\sigma_c} \star \phi)$.
    But $\proppart{B^\rho_c}$ is deductively closed by \closure{}, and
    $\proppart{B^\sigma_c}
    \star \phi$ is deductively closed by construction. Hence
    $\proppart{B^\rho_c} = \proppart{B^\sigma_c} \star \phi$, as required for
    \agm{}.

\end{proof}

As a consequence of \cref{prop:kconj_conditioning_implies_score_based} (and the
construction of $d$ in its proof), one can apply
\cref{lemma:score_based_agm_sufficient_conditions} with $M = 0$ for
conditioning operators with \kconj{} and a certain natural property.

\begin{corollary}
    \label{cor:conditioning_agm}
    Suppose an elementary conditioning operator satisfying \kconj{} has the
    property that
    \[
        W \in \X_{\tuple{\ast, c, \phi}} \iff W, c \models \phi
    \]
    Then \agm{} holds.
\end{corollary}

We can now prove \cref{prop:examples_satisfy_agm}.

\begin{proof}[Proof of \cref{prop:examples_satisfy_agm}]
    For the conditioning operators \varbasedcond{} and \partbasedcond{}, it is
    easily verified that the condition in \cref{cor:conditioning_agm} holds,
    and thus \agm{} does also.

    For the score-based operator \scorebasedop{}, we may use
    \cref{lemma:score_based_agm_sufficient_conditions} with $M = 0 $.
\end{proof}

\subsection{Proof of \cref{prop:examples_satisfy_condsucc}}

As a first step in the proof, we present sufficient conditions for conditioning
operators to satisfy \condsucc{}. In fact, we do not need to impose any
condition on the total preorder $\le$: a natural constraint on the mapping
$\sigma \mapsto \X_\sigma$ (together with some basic postulates) is enough.

\begin{lemma}
    \label{lemma:conditioning_condsucc_sufficient_condition}
    Suppose an elementary conditioning operator satisfies \kconj{},
    \soundness{} and
    \[
        W, c \models \phi \implies W \in \X_{\tuple{i, c, \phi}}
    \]
    Then \condsucc{} holds.
\end{lemma}

\begin{proof}
    Suppose an elementary conditioning operator corresponding to the mapping
    $\sigma \mapsto \tuple{\X_\sigma, \Y_\sigma}$ and total preorder $\le$
    satisfies \kconj{}, \soundness{} and has the stated property.

    Let $\sigma$ be a sequence and $c \in \C$. Suppose $E_i(\phi) \in
    B^\sigma_c$ and $\neg\phi \notin B^\sigma_c$. Write $\rho = \sigma \concat
    \tuple{i, c, \phi}$. We need to show $\phi \in B^\rho_c$.

    By $\neg\phi \notin B^\sigma_c$, there is some $W \in \Y_\sigma$ such that
    $W, c \models \phi$. Hence $W \in \X_{\tuple{i, c, \phi}}$. By
    elementariness and \kconj{}, we have $\X_\rho = \X_\sigma \cap
    \X_{\tuple{i, c, \phi}}$. Since $W \in \Y_\sigma \subseteq \X_\sigma$, we
    get $W \in \X_\rho$.

    Now take any $W' \in \Y_\rho$. Then $W'$ is $\le$-minimal in $\X_\rho$, so
    $W' \le W$. But $W$ is $\le$-minimal in $\X_\sigma$, so $W' \in \Y_\rho
    \subseteq \X_\rho \subseteq \X_\sigma$ gives $W' \in \Y_\sigma$ also.
    Consequently, $E_i(\phi) \in B^\sigma_c$ means $W', c \models E_i(\phi)$.
    On the other hand, \soundness{} together with $\tuple{i, c, \phi} \in \rho$
    and $W' \in \X_\rho$ means $W', c \models S_i(\phi)$. Hence $W', c \models
    E_i(\phi) \land S_i(\phi)$. From \cref{prop:validities}
    part (\labelcref{item:e_and_s_implies_phi}), we get $W', c \models \phi$.

    We have shown that $\phi$ holds in case $c$ at an arbitrary world in
    $\Y_\rho$. Hence $\phi \in B^\rho_c$, as required.

\end{proof}

Similarly, we have sufficient conditioning for score-based operators to satisfy
\strongcondsucc{}: the postulate follows if worlds in which $i$ makes a expert,
truthful report are strictly more plausible than worlds in which $i$ makes a
false report.

\begin{lemma}
    \label{lemma:score_based_strongcondsucc_sufficient_conditions}
    Suppose a score-based operator is such that for any $i \in \S$, $c \in \C$,
    $\phi \in \lprop$ and $W, W' \in \W$,
    \begin{align*}
        W, c &\models E_i(\phi) \land \phi
        \text{ and }
        W', c \models \neg\phi \\
        &\implies
        d(W, \tuple{i, c, \phi}) < d(W', \tuple{i, c, \phi})
    \end{align*}
    Then \strongcondsucc{} holds.
\end{lemma}

\begin{proof}
    Suppose a score-based operator has the stated property. Take $\sigma$ such
    that $\neg(E_i(\phi) \land \phi) \notin B^\sigma_c$. Write $\rho = \sigma
    \concat \tuple{i, c, \phi}$. We need to show that $\phi \in B^\rho_c$.

    First note that by $\neg(E_i(\phi) \land \phi) \notin B^\sigma_c$ and the
    definition of $B^\sigma$ for score-based operators, there is $W \in
    \Y_\sigma$ such that $W, c \models E_i(\phi) \land \phi$.

    Take any $W' \in \Y_\rho$. Suppose, for the sake of contradiction, that
    $W', c \not\models \phi$. Then by the hypothesised property of the score
    function $d$, we have
    \[
        d(W, \tuple{i, c, \phi}) < d(W', \tuple{i, c, \phi})
    \]
    Now, $W \in \Y_\sigma$ and $W' \in \Y_\rho \subseteq \X_\rho \subseteq
    \X_\sigma$ gives $r_\sigma(W) \le r_\sigma(W')$. Thus
    \begin{align*}
        r_\rho(W)
        &= r_\sigma(W) + d(W, \tuple{i, c, \phi}) \\
        &\le r_\sigma(W') + d(W, \tuple{i, c, \phi}) \\
        &< r_\sigma(W') + d(W', \tuple{i, c, \phi}) \\
        &= r_\rho(W') < \infty
    \end{align*}
    i.e. $r_\rho(W) < r_\rho(W') < \infty$. But this means $W \in \X_\rho$ and
    $W'$ is not minimal in $\X_\rho$ under $r_\rho$, contradicting $W' \in
    \Y_\rho$. Hence $W', c \models \phi$.

    Since $W'$ was an arbitrary member of $\Y_\rho$, we have shown $\phi \in
    B^\rho_c$, and thus \strongcondsucc{} is shown.
\end{proof}

The main result now follows.

\begin{proof}[Proof of \cref{prop:examples_satisfy_condsucc}]
    For the conditioning operators \varbasedcond{} and \partbasedcond{},
    \condsucc{} follows from
    \cref{lemma:conditioning_condsucc_sufficient_condition} since $W, c \models
    \phi$ implies $W, c \models S_i(\phi)$. For the score-based operator
    \scorebasedop{}, one can easily check that the condition in
    \cref{lemma:score_based_strongcondsucc_sufficient_conditions} holds, and
    thus \strongcondsucc{} and \condsucc{} follow.
\end{proof}

\subsection{Proof of \cref{prop:strongcondsucc_conditioning_impossibilitity}}

\begin{proof}

    Take distinct sources $i_1, i_2 \in \S \setminus \{\ast\}$, distinct cases
    $c, d \in \C$, and distinct valuations $v_1, v_2 \in \vals$. Let $\phi_1,
    \phi_2 \in \lprop$ be propositional formulas with $\propmods(\phi_k) = v_k$
    ($k \in \{1, 2\}$). Suppose for contradiction that some elementary
    conditioning operator -- satisfying the basic postulates -- has the stated
    properties.

    Define a sequence
    \[
        \sigma
        = (
            \tuple{\ast, c, \phi_1 \lor \phi_2},
            \tuple{i_1, c, \phi_1},
            \tuple{i_2, c, \phi_2}
        ).
    \]
    Let $\Pi_\bot$ denote the unit partition $\{\{u\} \mid u \in \vals\}$, and
    let $\widehat{\Pi}$ denote the partition
    \[
         \{\{v_1, v_2\}\}
         \cup
         \{\{u\} \mid u \in \vals \setminus \{v_1, v_2\}\},
    \]
    i.e. the partition obtained from $\Pi_\bot$ by merging the cells of $v_1$
    and $v_2$.

    Consider worlds $W_1$, $W_2$ given by
    \begin{align*}
         v^{W_k}_{c'} &= v_k \qquad (c' \in \C) \\
         \Pi^{W_k}_i &= \begin{cases}
            \widehat{\Pi},&
                (k = 1 \text{ and } i = i_2)
                \text { or }
                (k = 2 \text{ and } i = i_1) \\
            \Pi_\bot,& \text{ otherwise }
         \end{cases}
    \end{align*}
    That is, $W_1$ has $v_1$ as its valuation for all cases, $i_2$ has
    partition $\widehat{\Pi}$, and all other sources have the finest partition
    $\Pi_\bot$; similarly $W_2$ has $v_2$ for its valuations and all sources
    except $i_1$ have $\Pi_\bot$.

    Let $\le$ denote the total preorder associated with the conditioning
    operator.

    \begin{addmargin}[1em]{1em}
        \begin{claim}
            \label{claim:w1_simeq_w2}
            $W_1 \simeq W_2$.
        \end{claim}
        \begin{proof}
            Let $\pi$ be the permutation of $\S$ which swaps $i_1$ and $i_2$.
            It is easily observed that $\pi(W_1)$ is partition-equivalent to
            $W_2$. By reflexivity of partition refinement, $\pi(W_1) \preceq
            W_2$ and $W_2 \preceq \pi(W_1)$. By \refinement{}, we get $\pi(W_1)
            \simeq W_2$. By property (\labelcref{item:anonymity}), $W_1 \simeq
            \pi(W_1)$. By transitivity of ${\simeq}$ we get $W_1 \simeq W_2$ as
            desired.
        \end{proof}
    \end{addmargin}

    Now, from the basic postulates, property
    (\labelcref{item:conditioning_impossibility_first}) and
    \cref{prop:prior_knowledge} we have $K^\sigma = \cn(G^\sigma_\snd)$. By
    elementariness and \cref{lemma:model_based_elementary}, we get
    $\X_\sigma = \mods(K^\sigma) = \mods(G^\sigma_\snd)$. It is easily checked
    that both $W_1$ and $W_2$ satisfy the soundness statements corresponding to
    $\sigma$, and thus $W_1, W_2 \in \mods(G^\sigma_\snd) = \X_\sigma$.

    \begin{addmargin}[1em]{1em}
        \begin{claim}
            \label{claim:w1_w2_in_ysigm}
            $W_1, W_2 \in \Y_\sigma$.
        \end{claim}
        \begin{proof}
            We show $W_1$ and $W_2$ are $\le$-minimal in $\X_\sigma$. Take any
            $W \in \X_\sigma$. Then $W \in \mods(G^\sigma_\snd)$, so $W, c
            \models S_\ast(\phi_1 \lor \phi_2)$, i.e. $V^W_c \in \{v_1, v_2\}$.
            We consider two cases.
            \begin{itemize}
                \item \textbf{Case 1} ($v^W_c = v_1$). By $W \in
                    \mods(G^\sigma_\snd)$ again we have $W, c \models
                    S_{i_2}(\phi_2)$, i.e
                    \[
                        v_1
                        = v^W_c
                        \in \Pi^W_{i_2}[\phi_2]
                        = \Pi^W_{i_2}[v_2].
                    \]
                    It follows that $\{v_1, v_2\} \subseteq \Pi^W_{i_2}[v_2]$,
                    and that $\widehat{\Pi}$ refines $\Pi^W_{i_2}$. Since
                    $\widehat{\Pi}$ is the partition of $i_2$ in $W_1$, and all
                    other sources have the finest partition $\Pi_\bot$, we get
                    $W_1 \preceq W$. By \refinement{}, $W_1 \le W$. Since $W_1
                    \simeq W_2$ we have $W_2 \le W$ also.

                \item \textbf{Case 2} ($v^W_c = v_2$). Applying a near-identical
                      argument to that used in case 1 with soundness of the
                      report $\tuple{i_1, c, \phi_1}$, we get $W_1, W_2 \le W$.
            \end{itemize}
            In either case, both $W_1 \le W$ and $W_2 \le W$, so $W_1, W_2 \in
            \Y_\sigma$.
        \end{proof}
    \end{addmargin}

    Now we consider case $d$. Since
    \[
        W_1, d \models E_{i_1}(\phi_1) \land \phi_1
    \]
    and $W_1 \in \Y_\sigma$, $\neg(E_{i_1}(\phi_1) \land \phi_1) \notin
    B^\sigma_d$. Writing $\rho = \sigma \concat \tuple{i_1, d, \phi_1}$, we get
    from \strongcondsucc{} that $\phi_1 \in B^\rho_d$.

    Note that $W_2, d \models S_{i_1}(\phi_1)$, so $W_2 \in \mods(G^\rho_\snd)
    = \mods(K^\rho) = \X_\rho$. Since $W_2$ is $\le$-minimal in $\X_\sigma$ and
    \[
        X_\rho = \mods(G^\rho_\snd) \subseteq \mods(G^\sigma_\snd) = \X_\sigma,
    \]
    $W_2$ is also $\le$-minimal in $\X_\rho$, i.e. $W_2 \in \Y_\rho$. Now
    $\phi_1 \in B^\rho_d$ gives $W_2, d \models \phi_1$. Since $v^{W_2}_d =
    v_2$ and $\propmods(\phi_1) = \{v_1\}$, this means $v_1 = v_2$. But $v_1$
    and $v_2$ were assumed to be distinct: contradiction.

\end{proof}

\subsection{Proof of \cref{thm:selectivity_characterisation}}

\begin{proof}
    ``if": Suppose a model-based operator satisfies \boundedness{}. Take any
    $\ast$-consistent $\sigma$. For $c \in \C$, set
    \[
        M_c = \propmods\proppart{B^\sigma_c}.
    \]
    By \boundedness{}, we have $M_c \supseteq \propmods(\Gamma^\sigma_c)$. Now
    set
    \[
        F_\sigma(i, c, \phi) = \propmods(\phi) \cup M_c.
    \]
    Define a selection function $f_\sigma$ by letting $f_\sigma(i, c, \phi)$ be
    any formula with $\propmods(f_\sigma(i,c, \phi)) = F_\sigma(i, c, \phi)$.
    Since $F_\sigma(i, c, \phi)$ contains the models of $\phi$, clearly
    $f_\sigma(i, c, \phi) \in \cnprop(\phi)$. Therefore $f$ is indeed a
    selection function.

    We claim that, for any $c \in \C$,
    \[
        M_c
        = \bigcap_{\tuple{i, \phi} \in \sigma \rs c}{
            F_\sigma(i, c, \phi)
        }.
    \]
    The ``$\subseteq$" inclusion is clear since, by definition, $F_\sigma(i, c,
    \phi) \supseteq M_c$. For the ``$\supseteq$" inclusion, suppose for
    contradiction that there is some $v \in \bigcap_{\tuple{i, \phi} \in \sigma
    \rs c}{F_\sigma(i, c, \phi)}$ with $v \notin M_c$.

    Take any $\phi \in \Gamma^\sigma_c$. Then there is $i \in \S$ such that
    $\tuple{i, \phi} \in \sigma \rs c$, and hence $v \in F_\sigma(i, c, \phi)$.
    But $v \notin M_c$ by assumption, so $v \in \propmods(\phi)$. This shows $v
    \in \propmods(\Gamma^\sigma_c)$. But $\propmods(\Gamma^\sigma_c) \subseteq
    M_c$ by \boundedness{}, so $v \in M_c$; contradiction.

    From this we get
    \begin{align*}
        \propmods\proppart{B^\sigma_c}
        &= M_c \\
        &= \bigcap_{\tuple{i, \phi} \in \sigma \rs c}{F_\sigma(i, c, \phi)} \\
        &= \bigcap_{\tuple{i, \phi} \in \sigma \rs c}{
            \propmods(f_\sigma(i, c, \phi))
        } \\
        &= \propmods\left(\{
            f_\sigma(i, c, \phi) \mid \tuple{i, \phi} \in \sigma \rs c
        \}\right)
    \end{align*}
    Since $\proppart{B^\sigma_c}$ is deductively closed (by \closure{}, which
    holds for all model-based operators), we get
    \[
        \proppart{B^\sigma_c}
        =
        \cnprop\left(\{
            f_\sigma(i, c, \phi) \mid \tuple{i, \phi} \in \sigma \rs c
        \}\right)
    \]
    as required for selectivity.

    ``only if": Suppose a model-based operator is selective according to some
    selection scheme $f$. Take any $\ast$-consistent $\sigma$ and $c \in \C$.
    Write
    \[
        \Delta = \{
            f_\sigma(i, c, \phi)
            \mid
            \tuple{i, \phi} \in \sigma \rs c
        \}.
    \]
    so that $\proppart{B^\sigma_c} = \cnprop(\Delta)$. For $\tuple{i, \phi} \in
    \sigma \rs c$ we have $f_\sigma(i, c, \phi) \in \cnprop(\phi) \subseteq
    \cnprop(\Gamma^\sigma_c)$ from the definition of a selection scheme and the
    fact that $\phi \in \Gamma^\sigma_c$. Hence $\Delta \subseteq
    \cnprop(\Gamma^\sigma_c)$, so
    \[
        \proppart{B^\sigma_c}
        = \cnprop(\Delta)
        \subseteq \cnprop(\cnprop(\Gamma^\sigma_c))
        = \cnprop(\Gamma^\sigma_c)
    \]
    as required for \boundedness{}.
\end{proof}

\subsection{
    Proof that \varbasedcond{}, \partbasedcond{} and \scorebasedop{} are selective
}

This characterisation in \cref{thm:selectivity_characterisation} allows us to
easily analyse when conditioning and score-based operators are selective. In
the case of conditioning operators with $K^\emptyset = \cn(\emptyset)$, we in
fact have a precise characterisation. First, some terminology: say that a world
$W$ \emph{refines $W'$ at $c$} if for all $i \in \S$ we have $\Pi^W_i[v^W_c]
\subseteq \Pi^{W'}_i[v^{W'}_c]$. Intuitively, this means each source is more
knowledgable in case $c$ in world $W$ than they are in $W'$. Write $\W_{c\ :\
v} = \{W \in \W \mid v^W_c = v\}$ for the set of worlds whose $c$ valuation is
$v$. We have the following.

\begin{proposition}
    \label{prop:conditioning_selectivity_characterisation}
    Suppose an elementary conditioning operator satisfies the basic postulates
    and has $K^\emptyset = \cn(\emptyset)$. Then it is selective if and only if
    for all $W$, $c$, $v$ there is $W' \in \W_{c\ :\ v}$ such that $W' \le W$
    and $W$ refines $W'$ at all cases $d \ne c$.
\end{proposition}

While the condition on $\le$ in
\cref{prop:conditioning_selectivity_characterisation} is somewhat technical, it
is implied by the very natural \emph{partition-equivalence} property from
\cref{sec:framework}. Consequently, \varbasedcond{} and \partbasedcond{} are
selective. For the score-based operator \scorebasedop{}, one can show
\boundedness{} holds directly using a property of the disagreement scoring
function $d$ similar to the property of $\le$ above. Consequently,
\scorebasedop{} is also selective.

We first state some preliminary results.

\begin{lemma}
    \label{lemma:local_refinement}
    Suppose $W$ refines $W'$ at $c$. Then for any $i \in \S$ and $\phi \in
    \lprop$,
    \[
        W, c \models S_i(\phi)
        \implies
        W', c \models S_i(\phi)
    \]
\end{lemma}

\begin{proof}
    Suppose $W, c \models S_i(\phi)$. Then $v^W_c \in \Pi^W_i[\phi]$, i.e.
    $\propmods(\phi) \cap \Pi^W_i[v^W_c] \ne \emptyset$. By refinement,
    $\Pi^W_i[v^W_c] \subseteq \Pi^{W'}_i[v^{W'}_c]$. Hence $\propmods(\phi)
    \cap \Pi^{W'}_i[v^{W'}_c] \ne \emptyset$, so $v^{W'}_c \in
    \Pi^{W'}_i[\phi]$. That is, $W', c \models S_i(\phi)$.
\end{proof}

\begin{lemma}
    \label{lemma:local_refinement_sequence}
    For any $W \in \W$ and $c \in \C$, there is a $\ast$-consistent sequence
    $\sigma$ -- containing only reports for case $c$ -- such that for all $W'
    \in \W$,
    \[
        W' \in \mods(G^\sigma_\snd)
        \iff
        W \text{ refines } W' \text{ at } c.
    \]
\end{lemma}

\begin{proof}

    For a valuation $v \in \vals$, let $\phi(v)$ be a propositional formula
    such that $\propmods(\phi(v)) = \{v\}$. Take $\sigma$ to be any enumeration
    of reports of the form
    \[
        \tuple{i, c, \phi(v)}
    \]
    where $i \in \S$ and $v \in \Pi^W_i[v^W_c]$. Note that such a sequence
    exists since there are only finitely many sources and valuations. Clearly
    $\sigma$ contains only $c$-reports. Since $\Pi^W_\ast$ is the unit
    partition, the only report from $\ast$ is $\tuple{\ast, c, \phi(v^W_c)}$.
    Hence $\sigma$ is $\ast$-consistent. We show the desired equivalence.

    $\implies$: Suppose $W' \in \mods(G^\sigma_\snd)$. Take any $i \in \S$. We
    need to show $\Pi^W_i[v^W_c] \subseteq \Pi^{W'}_i[v^{W'}_c]$. Take $v \in
    \Pi^W_i[v^W_c]$. By construction of $\sigma$, $\tuple{i, c, \phi(v)} \in
    \sigma$. Hence $W', c \models S_i(\phi(v))$, i.e. $v^{W'}_c \in
    \Pi^{W'}_i[\phi(v)] = \Pi^{W'}_i[v]$. This shows $v \in
    \Pi^{W'}_i[v^{W'}_c]$ as required.

    $\impliedby$: Suppose $W$ refines $W'$ at $c$. Take any $\tuple{i, c,
    \phi(v)} \in \sigma$. Then $v \in \Pi^W_i[v^W_c]$, so $v^W_c \in \Pi^W_i[v]
    = \Pi^W_i[\phi(v)]$. This shows $W, c \models S_i(\phi(v))$, and
    \cref{lemma:local_refinement} gives $W', c \models S_i(\phi(v))$. Hence $W'
    \in \mods(G^\sigma_\snd)$.

\end{proof}

\begin{proof}[Proof of \cref{prop:conditioning_selectivity_characterisation}]
    Take an elementary conditioning operator with the basic postulates
    and $K^\emptyset = \cn(\emptyset)$.

    ``if":  Suppose the stated property holds. Since all conditioning operators
    are model-based, by \cref{thm:selectivity_characterisation} it suffices to
    show \boundedness{}. To that end, let $\sigma$ be $\ast$-consistent and
    take $c \in \C$. We need $\proppart{B^\sigma_c} \subseteq
    \cnprop(\Gamma^\sigma_c)$; or equivalently, by \closure{},
    $\propmods\proppart{B^\sigma_c} \supseteq \propmods(\Gamma^\sigma_c)$.

    Take any $v \in \propmods(\Gamma^\sigma_c)$. Since $\sigma$ is
    $\ast$-consistent, $B^\sigma$ is consistent by \consistency{}. Hence
    $\Y_\sigma \ne \emptyset$. Take any $W \in \Y_\sigma$. By the property in
    the statement of the result, there is $W' \in \W_{c\ :\ v}$ such that $W'
    \le W$ and $W$ refines $W'$ at all cases $d \ne c$.

    We claim $W' \in \X_\sigma$. By \cref{prop:prior_knowledge}, elementariness
    and \cref{lemma:model_based_elementary}, we have $\X_\sigma =
    \mods(K^\sigma) = \mods(G^\sigma_\snd)$. Take any $\tuple{i, d, \phi} \in
    \sigma$. We consider cases.
    \begin{itemize}
        \item \textbf{Case 1} ($d = c$). Here $\tuple{i, \phi} \in \sigma \rs
            c$, so $\phi \in \Gamma^\sigma_c$. Hence $v \in
            \propmods(\Gamma^\sigma_c) \subseteq \propmods(\phi)$. Since $W'
            \in \W_{c\ :\ v}$, $v$ is the $c$-valuation of $W'$. Hence $W', c
            \models \phi$, and $W', c \models S_i(\phi)$ follows.

        \item \textbf{Case 2} ($d \ne c$). By assumption, $W$ refines $W'$ at
              $d$. Since $W \in \Y_\sigma \subseteq \X_\sigma$, we have $W, d
              \models S_i(\phi)$. By \cref{lemma:local_refinement}, $W', d
              \models S_i(\phi)$ also.
    \end{itemize}
    We have shown $W' \in \mods(G^\sigma_\snd) = \X_\sigma$. Now recall that $W
    \in \Y_\sigma$ -- so $W$ is $\le$-minimal in $\X_\sigma$ -- and $W' \le W$.
    Thus $W'$ is also $\le$-minimal in $\X_\sigma$, i.e. $W' \in \Y_\sigma$.
    Since $W' \in \W_{c\ :\ v}$ also, we have by
    \cref{lemma:model_based_models_of_proppart} that $v \in
    \propmods\proppart{B^\sigma_c}$, as required.

    ``only if": Suppose our operator is selective, i.e. satisfies
    \boundedness{}. To show the desired property holds, take any $W$, $c$ and
    $v$. Enumerate $\C \setminus \{c\}$ as $\{d_1, \ldots, d_N\}$. By
    \cref{lemma:local_refinement_sequence}, for each $1 \le n \le N$ there is a
    $\ast$-consistent sequence $\sigma_n$ such that
    \[
        \mods(G^{\sigma_n}_\snd)
        =
        \{W' \in \W \mid W \text{ refines } W' \text{ at } d_n\}.
    \]
    Now, let $\phi$ and $\psi$ be formulas with $\propmods(\phi) = \{v\}$ and
    $\propmods(\psi) = \{v^W_c\}$. Let $\rho$ be the concatenation
    \[
        \rho
        =
        \sigma_1 \cdots \sigma_n \concat
        \tuple{\ast, c, \phi \lor \psi}.
    \]
    Note that $\rho$ is $\ast$-consistent, since each $\sigma_n$ is (and only
    refers to case $d_n$). We may therefore apply \boundedness{} for case $c$.
    Taking models of both sides yields
    \[
        \propmods\proppart{B^\rho_c}
        \supseteq \propmods(\Gamma^\rho_c)
        = \propmods(\phi \lor \psi)
        = \{v, v^W_c\}.
    \]
    In particular, $v \in \propmods\proppart{B^\rho_c}$. By
    \cref{lemma:model_based_models_of_proppart}, there is some $W' \in \Y_\rho
    \cap \W_{c\ :\ v}$.

    We show $W'$ has the required properties. First note that since $W$ refines
    itself at each $d_n$, we have $W \in \mods(G^{\sigma_n}_\snd)$. Clearly $W,
    c \models \psi$, so $W, c \models S_\ast(\phi \lor \psi)$ too. Thus $W \in
    \mods(G^\rho_\snd) = \X_\rho$ (using $K^\emptyset = \cn(\emptyset)$). Since
    $W' \in \Y_\rho = \min_{\le}{\X_\rho}$, we get $W' \le W$ as required.

    Next, take any case $d \ne c$. Then there is some $n$ such that $d = d_n$.
    Since $W' \in \Y_\rho \subseteq \X_\rho = \mods(G^\rho_\snd) \subseteq
    \mods(G^{\sigma_n}_\snd)$, we get that $W$ refines $W'$ at $d$. This
    completes the proof.

\end{proof}

\subsection{Proof of \cref{thm:case_independent_selectivity_characterisation}}

\begin{proof}

    ``only if": Suppose a model-based operator is selective
    according to some case-independent scheme $f$. Take any $\ast$-consistent
    $\sigma$, $H \subseteq \C$ and $c \in \C$. For any case $d$, write $M_d =
    \propmods\proppart{B^\sigma_d}$. Note that with $c_0$ an
    arbitrary fixed case, and writing $F_\sigma(i, \phi) =
    \propmods(f_\sigma(i, c_0, \phi))$, we have by case-independent-selectivity
    that
    \[
        M_d = \bigcap_{\tuple{i, \phi} \in \sigma \rs d}{
            F_\sigma(i, \phi)
        }.
    \]
    By closure, it is sufficient for \hboundedness{} to show that
    \begin{equation}
        \label{eqn:hbnd_inclusion}
        M_c
        \supseteq
        \propmods(\Gamma^{\sigma, H}_c) \cap \bigcap_{d \in H}{M_d}.
    \end{equation}
    Take any $v$ in the set on the right-hand side. To show $v \in M_c$, take
    any $\tuple{i, \phi} \in \sigma \rs c$. If $\phi \in \Gamma^{\sigma, H}_c$,
    then clearly
    \begin{align*}
        v
        &\in \propmods(\Gamma^{\sigma, H}_c) \\
        &\subseteq \propmods(\phi) \\
        &\subseteq \propmods(f_\sigma(i, c, \phi)) \\
        &= F_\sigma(i, \phi)
    \end{align*}
    (where we use $f_\sigma(i, c, \phi) \in \cnprop(\phi)$). Otherwise, $\phi
    \notin \Gamma^{\sigma, H}_c$. Since $\tuple{i, \phi} \in \sigma \rs c$,
    this means there is $d \in H$ such that $\tuple{i, \phi} \in \sigma \rs d$.
    Hence $v \in M_d$ gives $v \in F_\sigma(i, \phi)$. This shows the inclusion
    in \labelcref{eqn:hbnd_inclusion}, and we are done.

    ``if": Suppose a model-based operator satisfies \hboundedness{}. Let
    $\sigma$ be a $\ast$-consistent sequence. As before, write $M_c$ for
    $\propmods\proppart{B^\sigma_c}$. For $i \in \S$ and $c \in \C$,
    write
    \[
        \C(i, \phi) = \{c \in \C \mid \tuple{i, \phi} \in \sigma \rs c\},
    \]
    and set
    \[
        F_\sigma(i, \phi)
        = \propmods(\phi) \cup \bigcup_{c \in \C(i, \phi)}{M_c}.
    \]
    Define $f$ by letting $f_\sigma(i, c, \phi)$ be any propositional formula
    with $\propmods(f_\sigma(i, c, \phi)) = F_\sigma(i, \phi)$. Then $f$ is a
    case-independent selection scheme. We show our operator is selective
    according to $f$; by closure of $\proppart{B^\sigma_c}$ for each $c$, it
    suffices to show
    \[
        M_c = \bigcap_{\tuple{i, \phi} \in \sigma \rs c}{F_\sigma(i, \phi)}.
    \]
    Fix $c$. For the left-to-right inclusion, suppose $v \in M_c$. Take any
    $\tuple{i, \phi} \in \sigma \rs c$. Then $c \in \C(i, \phi)$, so
    $F_\sigma(i, \phi) \supseteq M_c$ and thus $v \in F_\sigma(i, \phi)$ as
    required.

    For the right-to-left inclusion, suppose $v$ lies in the intersection. Set
    \[
        H = \{d \in \C \mid v \in M_d\}.
    \]
    Apply \hboundedness{} and taking the models of both sides, we obtain
    \begin{equation}
        \label{eqn:hbnd_application}
        M_c
        \supseteq
        \propmods(\Gamma^{\sigma, H}_c) \cap \bigcap_{d \in H}{M_d}.
    \end{equation}
    Clearly $v \in \bigcap_{d \in H}{M_d}$ by definition of $H$. Let $\phi \in
    \Gamma^{\sigma, H}_c$. Then there is $i \in \S$ such that
    $\tuple{i, \phi} \in \sigma \rs c$, and consequently $v \in F_\sigma(i,
    \phi)$. We claim $v \in \propmods(\phi)$. If not, by definition of
    $F_\sigma(i, \phi)$ we must have $v \in \bigcup_{d \in \C(i, \phi)}{M_d}$,
    i.e. there is $d \in \C$ such that $\tuple{i, \phi} \in \sigma \rs d$ and
    $v \in M_d$. On the one hand, $\phi \in \Gamma^{\sigma, H}_c$ implies $d
    \notin H$. On the other, $v \in M_d$ gives $d \in H$ directly by the
    definition of $H$: contradiction. This shows $v \in \propmods(\phi)$. Since
    $\phi$ was arbitrary, we have $v \in \propmods(\Gamma^{\sigma, H}_c)$. By
    \labelcref{eqn:hbnd_application} we get $v \in M_c$, and the proof is
    complete.

\end{proof}




\bibliographystyle{kr}
\bibliography{references}

\end{document}